\documentclass[conference]{IEEEtran}

\IEEEoverridecommandlockouts
\usepackage{caption}
\usepackage{subcaption}
\usepackage{graphicx}
\usepackage{epstopdf}
\usepackage{hyperref}
\usepackage{amsmath}
\usepackage{mathtools,amssymb}
\usepackage{amsthm}
\usepackage{bbm}
\usepackage{mathtools}
\usepackage{tabularx,algorithm}	
\usepackage{color}
\usepackage{algpseudocode}
\usepackage{bm}
\usepackage{tikz}
\usepackage{cite}
\usepackage{amsfonts}
\usepackage{dsfont}
\usepackage{amssymb}
\usepackage{algorithm}
\usepackage{epsfig}
\usepackage[labelformat=simple]{subcaption}

\newtheorem{theorem}{Theorem}

\newtheorem{lemma}{Lemma}

\newtheorem{assumption}{Assumption}

\usetikzlibrary{positioning}
\usetikzlibrary{decorations.pathreplacing}
\usetikzlibrary{shapes,arrows}

\def\BibTeX{{\rm B\kern-.05em{\sc i\kern-.025em b}\kern-.08em
    T\kern-.1667em\lower.7ex\hbox{E}\kern-.125emX}}
\begin{document}

\title{CPnP: Consistent Pose Estimator for Perspective-n-Point Problem with Bias Elimination 
}
%
\author{Guangyang Zeng$^1$,  Shiyu Chen$^1$, Biqiang Mu$^2$, Guodong Shi$^3$, and Junfeng Wu$^1$
	\thanks{$^1$G. Zeng, S. Chen, and J. Wu are with the School of Data Science, Chinese University of Hong Kong, Shenzhen, Shenzhen, P. R. China,
		\texttt{ zengguangyang@cuhk.edu.cn,shiyuchen@link.cuhk.edu.cn, junfengwu@cuhk.edu.cn}.}
	\thanks{$^2$B. Mu is with Key Laboratory of Systems and Control, Institute of Systems Science, Academy of Mathematics and Systems Science, Chinese Academy of Sciences, Beijing, P. R. China,
		\texttt{bqmu@amss.ac.cn}.}
	\thanks{$^3$G. Shi is with the Australian Center for Field Robotics, School of Aerospace, Mechanical and Mechatronic Engineering, The University of Sydney, NSW 2008, Sydney,
		\texttt{guodong.shi@sydney.edu.au}.}
	\thanks{The implementation code can be found at \url{https://github.com/SLAMLab-CUHKSZ/CPnP-A-Consistent-PnP-Solver}.}
}
%

%


\maketitle

\begin{abstract}
The Perspective-n-Point (PnP) problem has been widely studied in both computer vision and photogrammetry societies. With the development of feature extraction techniques, a large number of feature points might be available in a single shot. It is promising to devise a consistent estimator, i.e., the estimate can converge to the true camera pose as the number of points increases. To this end, we propose a consistent PnP solver, named \emph{CPnP}, with bias elimination. Specifically, linear equations are constructed from the original projection model via 
measurement model modification and variable elimination, based on which a closed-form least-squares solution is obtained. We then analyze and subtract the asymptotic bias of this solution, resulting in a consistent estimate. Additionally, Gauss-Newton (GN) iterations are executed to refine the consistent solution. Our proposed estimator is efficient in terms of computations---it has $O(n)$ computational complexity. Experimental tests on both synthetic data and real images show that our proposed estimator is superior to some well-known ones for images with dense visual features,
in terms of estimation precision and computing time.
\end{abstract}

%

\section{Introduction}
Given $n$ 3D points and their corresponding 2D projections on the image plane, inferring the pose (rotation and translation) of a calibrated camera is referred to as the Perspective-n-Point (PnP) problem. It has widespread applications in robotics~\cite{qin2018vins,meng2017backtracking,pumarola2017pl}, computer vision~\cite{vakhitov2021uncertainty}, augmented reality~\cite{zhou2018re}, etc. 

Most of the existing works set out from ideal geometric relationships to construct geometry-constrained equations. They do not explicitly take the projection noises into account and overlook the noise propagation in equation transformations. As such, they rarely analyze the property of the designed estimator from the perspective of statistics, such as the bias and covariance of the estimator, which are important metrics in estimation theory. 
It is noteworthy that with the development of feature extraction techniques, a large number of feature points might be available in a single shot. For example, Fig.~\ref{ETH3D_figures} in the experiment part are four images from ETH3D dataset~\cite{schops2017multi}, with two in the outdoor scenario and two in the indoor scenario. All images contain thousands of feature points, exhibiting the huge potential to yield a precise pose estimate. By noting this, we argue that it is promising to revisit the PnP problem through the lens of statistics.

The works that took the projection noises into account include the CEPPnP~\cite{ferraz2014leveraging}, MLPnP~\cite{urban2016mlpnp}, and EPnPU\cite{vakhitov2021uncertainty}. In these works, the covariance matrix of noises is utilized to improve the estimation precision. However, they did not analyze the statistical property of the proposed estimators. Actually, due to the nonlinear nature of projection equations, all of these estimators are biased, and thus not consistent. The definition of \emph{consistent} here is that the estimate can converge in probability to the true value as the number of points increases.
To the best of our knowledge, the design of a consistent PnP solver is still an open problem. 

In this paper, we devise a {\bf c}onsistent {\bf PnP} (CPnP) solver in virtue of bias elimination. Specifically, linear equations are constructed from the original projection model via measurement model modification and variable elimination, and a closed-form least-squares solution is obtained. The least-squares solution is biased, hence we further draw lessons from~\cite{mu2017globally} to give a consistent estimate of the variance of projection noises, based on which the asymptotic bias is eliminated, yielding a consistent estimate of the camera pose. Additionally, we perform constrained Gauss-Newton (GN) iterations to refine the consistent estimate. 
Our proposed PnP solver owns the attractive property that the estimate can converge to the true parameters as the number of points increases. In addition, the solver is efficient---its computational complexity is $O(n)$, which is superior to most of the state-of-the-art algorithms. These two properties make our estimator favorable in the presence of numerous feature points. 

We compare the proposed estimator with some well-known PnP solvers using both synthetic data and a benchmark dataset, named ETH3D~\cite{schops2017multi}. The results show that when the number of feature points exceeds several hundred, our estimator has minimal RMSEs and almost the same CPU time as the EPnP algorithm. In addition, in the case of large noise intensity, the covariance of the proposed estimator can converge to $0$, while the compared estimates have asymptotic biases, failing to converge to the true parameters. 

To summarize, the main contributions of this paper are two-fold: ($i$) We revisit the PnP problem from the perspective of statistics, finding that the existing PnP solvers are generally biased, thus not consistent; ($ii$) On the basis of a series of techniques, including variable elimination, noise estimation, and bias elimination, a consistent $O(n)$ estimator is proposed. 

\section{Related Works}
The last two decades have witnessed great progress on PnP solvers. Some works considered a fixed number of points. The minimal number of points to solve the PnP problem is three, and the prominent solutions to the P3P problem include~\cite{dementhon1992exact,kneip2011novel,li2011stable}. Apart from P3P solvers, there are also P4P~\cite{bujnak2008general} and P5P solutions~\cite{triggs1999camera}. These analytic methods may return inaccurate or erroneous solutions in the presence of noise. 
Most solvers, however, can deal with an arbitrary number of points. These generic PnP methods can be categorized into non-iterative and iterative ones. 

Early non-iterative methods are generally computationally complex, e.g.,~\cite{ansar2003linear} with $O(n^8)$,~\cite{quan1999linear} with $O(n^5)$, and~\cite{fiore2001efficient} with $O(n^2)$. The EPnP solution~\cite{lepetit2009epnp} is the first widely known $O(n)$ solver for the PnP problem and is utilized in many robot applications~\cite{boukas2015towards,lee2020camera,lambrecht2021optimizing}. Its main idea is to represent the 3D point coordinates as a linear combination of four control points. These control points are then retrieved using linearization techniques.
Apart from the linearization techniques, the polynomial solvers have become a mainstream, e.g., the RPnP~\cite{li2012robust}, DLS~\cite{hesch2011direct}, OPnP~\cite{zheng2013revisiting}, SRPnP~\cite{wang2018simple}, and EOPnP~\cite{zhou2019efficient}. These works all estimate the camera pose via solving polynomial equations with $O(n)$ complexity.

Iterative methods solve PnP optimization problems in an iterative manner. One main issue of iterative methods is the set of the initial estimate---they can achieve excellent precision when they converge properly. LHM~\cite{lu2000fast} and FP~\cite{pavlakos20176} initialized the estimate with a weak perspective approximation and refined the estimate via successive iterations. SQPnP~\cite{terzakis2020consistently} set several initials and then each regional minimum was computed with a sequential quadratic programming scheme. The aforementioned non-iterative solutions are often used as initial values, and iterative algorithms, e.g., Gauss-Newton (GN) iterations are applied to refine the estimate. For instance, EPnP~\cite{lepetit2009epnp} and MLPnP~\cite{urban2016mlpnp} both evaluated the performance by adding the GN iterations. 

It is noteworthy that most of the existing works have not modelled the projection noises, based on which the solver can be optimized accordingly. The literature that took projection noise  into account includes CEPPnP~\cite{ferraz2014leveraging}, MLPnP~\cite{urban2016mlpnp}, PLUM~\cite{li2017combining}, and~EPnPU\cite{vakhitov2021uncertainty}. In these works, the covariance matrix of noises is utilized to improve the estimation precision. However, there is little research on the statistical properties of 
their proposed estimators.
Actually, due to the inherent nonlinear nature of the camera projection model, all of these estimators are biased, and thus not consistent. In this paper, we take the noise model into account and devise an estimator with bias elimination. The proposed estimator is consistent, i.e., the estimate converges to the true value as the number of feature points increases.

\section{Problem Formulation}
{\bf Notations:} For a noisy quantity ${\bf a}$, we use ${\bf a}^o$ to denote its true value. For a vector ${\bf a}$, $[{\bf a}]_i$ is the $i$-th element of ${\bf a}$. ${\bf 1}_{i \times j}$ and ${\bf 0}_{i \times j}$ are $i \times j$ matrices whose elements are all $1$ and $0$. 

As shown in Fig.~\ref{PnP_scenario}, suppose there are $n$ points whose coordinates in the world frame $\{W\}$ are ${\bf p}_i^w=[x_{i}^w~y_{i}^w~z_{i}^w]^\top, i=1,\ldots,n$. Denote their coordinates in the camera frame $\{C\}$ as ${\bf p}_i^c=[x_{i}^c~y_{i}^c~z_{i}^c]^\top, i=1,\ldots,n$. Then, given the rotation matrix $\bf R$ and transformation vector $\bf t$ from the world frame to the camera frame, it holds that ${\bf p}_i^c = {\bf R} {\bf p}_i^w + {\bf t}$.
Further, for a calibrated pinhole camera with the intrinsic matrix
\begin{equation*}
{\bf K} = \begin{bmatrix}
f_x~~0~~u_0 \\
0~~f_y~~v_0 \\
0~~~0~~~1
\end{bmatrix},
\end{equation*}
the ideal projection equation is 
\begin{equation} \label{cv_projection_equation}
s_i \begin{bmatrix}
{\bf q}_i \\
1
\end{bmatrix} = 
{\bf K} [{\bf R}~{\bf t}] \begin{bmatrix}
{\bf p}_i^w \\
1
\end{bmatrix},
\end{equation}
where $s_i$ is the scale factor. From the third line of~\eqref{cv_projection_equation}, we have $s_i={\bf e}_3^\top ({\bf R} {\bf p}_i^w + {\bf t})$, where ${\bf e}_i$ is the unit vector whose $i$-th element is $1$. Considering projection noises,~\eqref{cv_projection_equation} can be rewritten as 
\begin{equation} \label{noise_free_projection}
{\bf q}_i = \frac{{\bf W}{\bf E} ({\bf R} {\bf p}_i^w + {\bf t})}{{\bf e}_3^\top ({\bf R} {\bf p}_i^w + {\bf t})} +\begin{bmatrix}
u_0 \\
v_0
\end{bmatrix} + {\bm \epsilon}_i,
\end{equation}
where ${\bf W}={\rm diag}(f_x,f_y)$, ${\bf E}=[{\bf e}_1~{\bf e}_2]^\top$, and ${\bm \epsilon}_i$ is the projection noise. For ${\bm \epsilon}_i$, we assume that 
\begin{assumption} \label{noise_assump}
	The measurement noises ${\bm \epsilon}_i \sim \mathcal N (0,\sigma^2 I),i=1,\ldots,n$ are i.i.d. with unknown variance $\sigma^2 < \infty$.
\end{assumption}
Assumption~\ref{noise_assump} has been widely used in simulations for the PnP problem, e.g.,~\cite{lepetit2009epnp,hesch2011direct,urban2016mlpnp}. Since the intrinsic matrix is known, for simplicity, we can obtain
\begin{equation} \label{projection_equation}
{\bf q}_i' =  {\bf q}_i - \begin{bmatrix}
u_0 \\
v_0
\end{bmatrix}   =\frac{{\bf W}{\bf E} ({\bf R} {\bf p}_i^w + {\bf t})}{{\bf e}_3^\top ({\bf R} {\bf p}_i^w + {\bf t})}+{\bm \epsilon}_i.
\end{equation}
The PnP problem is to estimate $\bf R$ and $\bf t$ from the $n$ correspondences between ${\bf p}_i^w$ and ${\bf q}_i'$. A prevalent criterion to do so is to minimize the sum of squared reprojection errors:
\begin{subequations}\label{LS_problem}
	\begin{align}
	\mathop{\rm minimize~}\limits_{{\bf R},{\bf t}} ~& \frac{1}{n} \sum_{i=1}^{n} \left\|  \frac{{\bf W}{\bf E} ({\bf R} {\bf p}_i^w + {\bf t})}{{\bf e}_3^\top ({\bf R} {\bf p}_i^w + {\bf t})} - {\bf q}_i' \right\|^2  \label{LS_objective} \\
	\mathop{\rm subject~to~} ~& {\bf R} \in {\rm SO}(3),\label{LS_constraint}
	\end{align}
\end{subequations}
The constrained least-squares problem~\eqref{LS_problem} is a nonconvex optimization problem whose global solution is hard to solve. In the following section, we relax the constraint~\eqref{LS_constraint} and design a consistent estimator for $\bf R$ and $\bf t$. 

	\begin{figure}[!htb]
	\centering
	\includegraphics[width=0.68\columnwidth]{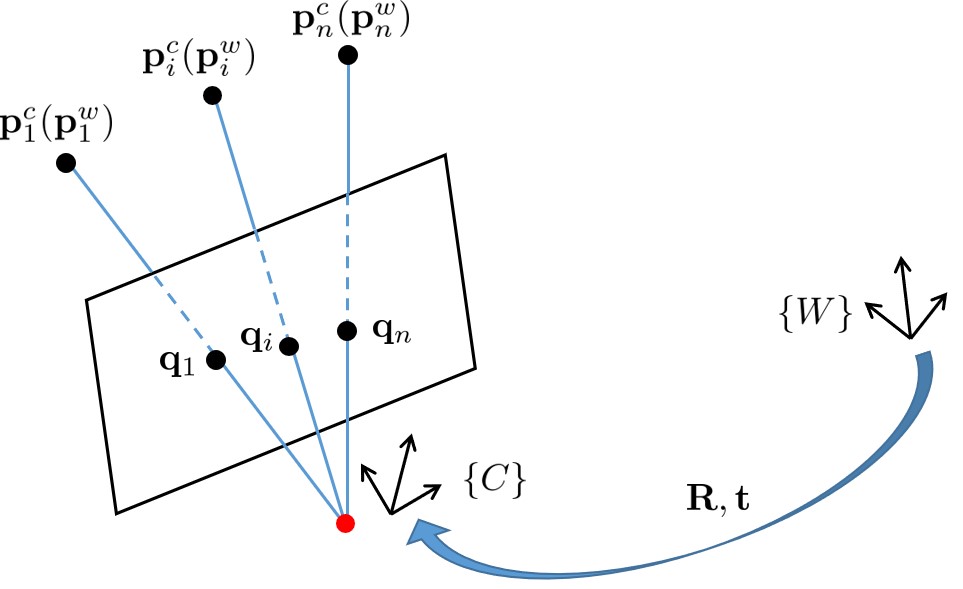}
	\caption{Illustration of the PnP problem.}
	\label{PnP_scenario}
\end{figure}

\section{Relaxed Least-Squares Solution: A Consistent Pose Estimate}
In this section, we will first modify the original measurement equations~\eqref{projection_equation} to obtain linear ones. We then conduct variable elimination to avoid scale ambiguity. Additionally, we give a consistent estimate of the covariance of projection noises, based on which a bias-eliminated least-squares solution is proposed. Further, the consistent estimate is refined with GN iterations. 

\subsection{Modified Measurement equations}
By multiplying both sides of~\eqref{projection_equation} by ${\bf e}_3^\top ({\bf R} {\bf p}_i^w + {\bf t})$ we obtain a modified linear measurement model: 
\begin{equation} \label{modified_measurement_equation}
{\bf W}{\bf E} ({\bf R} {\bf p}_i^w + {\bf t}) - {\bf e}_3^\top ({\bf R} {\bf p}_i^w + {\bf t}) {\bf q}_i' + {\bf e}_3^\top ({\bf R} {\bf p}_i^w + {\bf t}) {\bm \epsilon}_i=0,
\end{equation}
where ${\bf e}_3^\top ({\bf R} {\bf p}_i^w + {\bf t}) {\bm \epsilon}_i$ is the scaled noise term.
Note that~\eqref{modified_measurement_equation} is linear in terms of $\bf R$ and $\bf t$. Let $\bar {\bf L}_i=\left[{{\bf p}_i^w}^\top~1 \right]  \otimes {\bf I}_3 \in \mathbb R^{3 \times 12}$. By vectorizing $[{\bf R}~{\bf t}]$, i.e., ${\bf x} ={\rm vec} \left([{\bf R}~{\bf t}] \right) $, we can concatenating~\eqref{modified_measurement_equation} for all $n$ reference points to obtain the matrix form of the modified measurement 
model:
\begin{equation}\label{linear_equations}
{\bf 0}={\bf M} {\bf x}+{\bm \epsilon}^{\prime},
\end{equation}
where 
${\bf M}=[{\bf M}_1^\top~\cdots~{\bf M}_n^\top]^\top$, ${\bf M}_i=({\bf W}{\bf E}-{\bf q}_i'{\bf e}_3^\top) \bar{\bf L}_i$ and ${\bm \epsilon}^{\prime}=\left[ 
{\bf e}_3^\top ({\bf R} {\bf p}_1^w + {\bf t}) {\bm \epsilon}_1^\top ~\cdots~{\bf e}_3^\top ({\bf R} {\bf p}_n^w + {\bf t}) {\bm \epsilon}_n^\top
\right] ^\top$. 

For the modified measurement equation~\eqref{linear_equations}, on the one hand, the regressand is $\bf 0$; on the other hand, due to the scale ambiguity, the regressor $\bf M$ is not full column rank. Therefore, the estimate of $\bf x$ cannot be calculated in a closed form.
In the following, we will conduct variable elimination to avoid scale ambiguity and make the regressand a nonzero vector. After this procedure, the number of variables to be estimated is reduced from $12$ to $11$.

\subsection{Variable Elimination}
To eliminate the scale ambiguity, we introduce the following constraint: 
\begin{equation}\label{variable_elimination}
\alpha \sum_{i=1}^{n} {\bf e}_3^\top ({\bf R} {\bf p}_i^w + {\bf t})=n,
\end{equation}
where $\alpha$ is the scale factor. 
Let ${\bf R}=[{\bf r}_1~{\bf r}_2~{\bf r}_3]^\top$, and ${\bf t}=[t_1~t_2~t_3]^\top$. Define $\bar {\bf p}^w :=\sum_{i=1}^{n} {\bf p}_i^w /n$. From~\eqref{variable_elimination} we have 
\begin{equation}\label{t3_expression}
\alpha t_3 = 1- \alpha {\bar {\bf p}{{}^w}}^\top {\bf r}_3.
\end{equation}
Substituting $t_3$ into~\eqref{modified_measurement_equation} yields 
\begin{equation}\label{variable_eliminate_equation}
{\bf q}_i'=\alpha {\bf W}{\bf E} ({\bf R} {\bf p}_i^w + {\bf t}) - \alpha \left({\bf p}_i^w- \bar {\bf p}^w \right) ^\top {\bf r}_3 {\bf q}_i'+ {\bm \varepsilon}_i,
\end{equation}
where ${\bm \varepsilon}_i=(1+\alpha \left({\bf p}_i^w- \bar {\bf p}^w \right)^\top {\bf r}_3) {\bm \epsilon}_i$.
By the above variable elimination, the number of unknown variables is reduced from $12$ to $11$. We stack the $11$ variables as follows:
\begin{equation}\label{11_variables}
{\rm \bm \theta}:= [\hat{\bf r}_3^\top~\hat{\bf r}_1^\top~\hat t_1~\hat {\bf r}_2^\top~\hat t_2]^\top =\alpha [{\bf r}_3^\top~{\bf r}_1^\top~t_1~{\bf r}_2^\top~t_2]^\top. 
\end{equation}

Given the vector $\bm \theta$, the rotation matrix $\bf R$ and transformation vector $\bf t$ along with the scale factor $\alpha$ can be uniquely recovered; see~\eqref{recover_alpha}-\eqref{recover_t}. This is due to ${\rm det}(\bf R)=1$ since ${\bf R} \in {\rm SO}(3)$.
The same variable elimination method is adopted in~\cite{zheng2013revisiting}. Compared with the prevalent strategy in DLT-based methods, i.e., setting the constraint $\|{\bf x}\|=1$, which leads to a nonlinear relationship among variables, the constraint in~\eqref{variable_elimination} owns the advantage that the resulting equation~\eqref{t3_expression} has a linear form. This facilitates the construction of the following linear system of equations.
Let ${\bf q}_i'=[u_i~v_i]^\top$, by concatenating~\eqref{variable_eliminate_equation} for all reference points, we obtain the following matrix form:
\begin{equation}\label{modified_matrix_form}
{\bf b}= {\bf A} {\bm \theta} + {\bm \varepsilon},
\end{equation}
where ${\bf b}=[{\bf q}_1'^\top~\cdots~{\bf q}_n'^\top]^\top$, ${\bm \varepsilon}=[{\bm \varepsilon}_1^\top~\cdots~{\bm \varepsilon}_n^\top]^\top$, and 
\begin{equation*}
{\bf A}=\begin{bmatrix}
-u_1 \left({\bf p}_1^w- \bar {\bf p}^w \right) ^\top~f_x{{\bf p}_1^w}^\top~f_x~{\bf 0}_{1 \times 4} \\
-v_1 \left({\bf p}_1^w- \bar {\bf p}^w \right) ^\top~{\bf 0}_{1 \times 4}~f_y{{\bf p}_1^w}^\top~f_y \\
\vdots \\
-u_n \left({\bf p}_n^w- \bar {\bf p}^w \right) ^\top~f_x{{\bf p}_n^w}^\top~f_x~{\bf 0}_{1 \times 4} \\
-v_n \left({\bf p}_n^w- \bar {\bf p}^w \right) ^\top~{\bf 0}_{1 \times 4}~f_y{{\bf p}_n^w}^\top~f_y
\end{bmatrix}.
\end{equation*}
Compared with~\eqref{linear_equations},~\eqref{modified_matrix_form} has the advantages that it is a nonhomogeneous system, and the matrix $\bf A$ has full column rank in general. Hence, a closed-form solution is given by 
\begin{equation}\label{biased_LS_solution}
\hat {\bm \theta}_n^{\rm B}=\left( {\bf A}^\top {\bf A}\right) ^{-1} {\bf A}^\top {\bf b}.
\end{equation}
It is noteworthy that the projections ${\bf q}_i'=[u_i~v_i]^\top$ contain noises and thus the regressor $\bf A$ and the regressand $\bf b$ are correlated, which leads to biasedness of $\hat {\bm \theta}_n^{\rm B}$. In the following subsection, we give a consistent estimate of the variance of projection noises, based on which the asymptotic bias of $\hat {\bm \theta}_n^{\rm B}$ can be eliminated.

\subsection{Least-Squares Solution with Bias Elimination}
First, we use the method in~\cite{mu2017globally} to obtain a consistent estimate of $\sigma^2$, denoted as ${\hat \sigma}^2$. The details are presented in Appendix A. Based on the consistent estimate of noise covariance, we are on the point to eliminate the bias of the solution given in~\eqref{biased_LS_solution}.
Specifically, define
\begin{equation*}
{\bf G}=\begin{bmatrix}
- \left({\bf p}_1^w- \bar {\bf p}^w \right) ^\top ~ {\bf 0}_{1 \times 8} \\
- \left({\bf p}_1^w- \bar {\bf p}^w \right) ^\top ~ {\bf 0}_{1 \times 8} \\
\vdots \\
- \left({\bf p}_n^w- \bar {\bf p}^w \right) ^\top ~ {\bf 0}_{1 \times 8} \\
-\left({\bf p}_n^w- \bar {\bf p}^w \right) ^\top ~ {\bf 0}_{1 \times 8} 
\end{bmatrix} \in \mathbb R^{2n \times 11}.
\end{equation*}
The bias-eliminated solution is given as
\begin{equation}\label{bias_eliminated_solution}
\hat {\bm \theta}_n^{\rm BE}=\left( {\bf A}^\top {\bf A}- {\hat \sigma}^2 {\bf G}^\top {\bf G}\right) ^{-1} \left( {\bf A}^\top {\bf b}-  {\hat \sigma}^2 {\bf G}^\top {\bf 1}_{2n \times 1} \right) .
\end{equation}
To ensure the consistency of $\hat {\bm \theta}_n^{\rm BE}$, we make the following assumption.
	\begin{assumption} \label{reference_point_assump}
		The reference points $({\bf p}_i^w)_{i=1}^{n}$ do not concentrate on any critical set given in Result 22.5 in~\cite{hartley2003multiple}. 
	\end{assumption}
Assumption~\ref{reference_point_assump} is an assumption on the asymptotic distribution of 3D feature points. It can be satisfied in general, guaranteeing that the camera matrix ${\bf K} [{\bf R}~{\bf t}]$ in~\eqref{cv_projection_equation} is unique~\cite{hartley2003multiple}. The following theorem presents the asymptotic property of the bias-eliminated solution $\hat {\bm \theta}_n^{\rm BE}$. 
\begin{theorem} \label{consistency_theorem}
Given Assumptions~\ref{noise_assump}-\ref{reference_point_assump}, the estimate $\hat {\bm \theta}_n^{\rm BE}$ is consistent, i.e., 
\begin{equation*}
    \mathop{\rm plim}\limits_{n \rightarrow \infty} ~\hat {\bm \theta}_n^{\rm BE}={\bm \theta}^o.
\end{equation*}
\end{theorem}
The proof of Theorem~\ref{consistency_theorem} is presented in Appendix B. 
Given $\hat {\bm \theta}_n^{\rm BE}$, according to~\eqref{11_variables}, the estimate of $\bf R$ and $\bf t$ can be calculated as follows:
{\small
\begin{align}
\hat \alpha & = \left( {\rm det}\left(\begin{bmatrix}
\left[ \hat {\bm \theta}_n^{\rm BE}\right] _{4:6}~~\left[ \hat {\bm \theta}_n^{\rm BE}\right] _{8:10}~~\left[ \hat {\bm \theta}_n^{\rm BE}\right] _{1:3}
\end{bmatrix}^\top \right) \right)^{1/3}, \label{recover_alpha}\\
\hat {\bf R}_n^{\rm BE} & = 
\begin{bmatrix}
\left[ \hat {\bm \theta}_n^{\rm BE}\right] _{4:6}~~\left[ \hat {\bm \theta}_n^{\rm BE}\right] _{8:10}~~\left[ \hat {\bm \theta}_n^{\rm BE}\right] _{1:3} 
\end{bmatrix}^\top \Big/ \hat \alpha, \label{recover_R}\\
\hat {\bf t}_n^{\rm BE} & = 
\begin{bmatrix}
\left[ \hat {\bm \theta}_n^{\rm BE}\right] _{7}~~\left[ \hat {\bm \theta}_n^{\rm BE}\right] _{11}~~1-{{}\bar {\bf p}^w}^\top \left[ \hat {\bm \theta}_n^{\rm BE}\right] _{1:3}
\end{bmatrix}^\top \Big/ \hat \alpha. \label{recover_t}
\end{align}}
Since $\hat {\bm \theta}_n^{\rm BE}$ is consistent, so are $\hat {\bf R}_n^{\rm BE}$ and $\hat {\bf t}_n^{\rm BE}$. 

Note that the matrix $\hat {\bf R}_n^{\rm BE}$ does not necessarily fall within ${\rm SO}(3)$. Hence, we should further project it into ${\rm SO}(3)$.
Let the SVD of $\hat {\bf R}_n^{\rm BE}$ be $\textbf{U}{\bf \Sigma}{\bf V}^\top$, then the following projection gives the closest element in ${\rm SO}(3)$ to $\hat {\bf R}_n^{\rm BE}$ in terms of the Frobenius norm~\cite{arun1987least}:
\begin{equation} \label{SVD_projection}
    \pi(\hat {\bf R}_n^{\rm BE})= \textbf{U}{\rm diag}([1~1~\det(\textbf{U}{\bf V}^{\rm T})]^{\rm T}) {\bf V}^{\rm T}.
\end{equation}
With a little abuse of notations, we still denote the projected matrix as $\hat {\bf R}_n^{\rm BE}$.

\subsection{Constrained Gauss-Newton Iteration}

With the consistent preliminary estimates $\hat {\bf R}_n^{\rm BE}$ and $\hat {\bf t}_n^{\rm BE}$, local methods such as the GN iteration associated with measurement equation~\eqref{projection_equation} can be applied to further improve the estimation precision. Let ${\bf L}_i={{\bf p}_i^w}^\top \otimes {\bf I}_3 \in \mathbb R^{3 \times 9}$. The measurement equation~\eqref{projection_equation} can be rephrased as
\begin{equation}\label{projection_equation2}
    {\bf q}_i' =  \frac{{\bf W} {\bf E} \left( {\bf L}_i {\rm vec}({\bf R})+{\bf t}\right)}{{\bf e}_3^\top \left( {\bf L}_i {\rm vec}({\bf R})+{\bf t}\right)}+{\bm \epsilon}_i.
\end{equation}
Note that the rotation matrix ${\bf R}$ needs to be in ${\rm SO}(3)$. Hence, we cannot directly use the Jacobian obtained by taking derivatives with respect to ${\bf R}$ in the GN iterations. Note that we can represent the rotation matrix ${\bf R}$ in the vicinity of a given matrix ${\bf R}_0$ as ${\bf R}= {\bf R}_0\exp({\bf s}^\wedge)$, where ${\bf R}_0\in{\rm SO}(3)$ and ${\bf s}^\wedge$ is a skew-symmetric matrix generated by the ``hat'' operation:
$${\bf s}^\wedge=\begin{bmatrix}
0 & -s_3 & s_2 \\
s_3 & 0 & -s_1 \\
-s_2 & s_1 & 0 
\end{bmatrix}
$$
with ${\bf s}=[s_1~s_2~s_3]^\top$. We calculate the Jacobian associated with $\bf s$ to guarantee the updated estimate of the rotation matrix is still in ${\rm SO}(3)$. Define 
\begin{equation*}
{\bm \Psi}:=\frac{\partial {\rm vec}(\exp(\bf s^{\wedge}))}{
	\partial {\bf s}^\top} \Big\rvert_{\bf s=0},
\end{equation*}
and 
\begin{equation*}
\begin{split}
f_{i}({\bf R},{\bf s},{\bf t})&:=\frac{{\bf W}{\bf E}\left( {\bf L}_i {\rm vec}({\bf R}\exp(\bf s^{\wedge}))+{\bf t}\right)}{{\bf e}_3^\top\left( {\bf L}_i {\rm vec}({\bf R}\exp(\bf s^{\wedge}))+{\bf t}\right)}, \\
g_{i}({\bf R},{\bf s},{\bf t})&:={\bf W}{\bf E}\left( {\bf L}_i {\rm vec}({\bf R}\exp(\bf s^{\wedge}))+{\bf t}\right), \\
h_{i}({\bf R},{\bf s},{\bf t})&:={\bf e}_3^\top\left( {\bf L}_i {\rm vec}({\bf R}\exp(\bf s^{\wedge}))+{\bf t}\right).
\end{split}
\end{equation*}
The derivatives are 
{\small
\begin{align*}
\frac{\partial f_{i}({\bf R},{\bf s},{\bf t})}{\partial {\bf s}^\top} \Big\rvert_{{\bf s}=0} =& \frac{(h_{i}({\bf R},0,{\bf t}) {\bf W}{\bf E} -g_{i}({\bf R},0,{\bf t}) {\bf e}_3^\top) {{\bf p}_i^w}^\top \otimes {\bf R} {\bm \Psi}}{h_{i}({\bf R},0,{\bf t})^2}\\
\frac{\partial f_{i}({\bf R},{\bf s},{\bf t})}{\partial {\bf t}^\top} \Big\rvert_{{\bf s}=0} =&\frac{h_{i}({\bf R},0,{\bf t}) {\bf W}{\bf E} -g_{i}({\bf R},0,{\bf t}){\bf e}_3^\top}{h_{i}({\bf R},0,{\bf t})^2}.
\end{align*}}
Then we can obtain the Jacobian matrix ${\bf J}({\bf R}, {\bf t})$ as follows:
\begin{equation*}
{\bf J}({\bf R}, {\bf t})=\begin{bmatrix}
\vdots ~~~~~~~~~~~~~~\vdots\\
\frac{\partial f_{i}({\bf R},{\bf s},{\bf t})}{\partial {\bf s}^\top} \Big\rvert_{{\bf s}=0}~~\frac{\partial f_{i}({\bf R},{\bf s},{\bf t})}{\partial {\bf t}^\top} \Big\rvert_{{\bf s}=0}\\
\vdots ~~~~~~~~~~~~~~\vdots
\end{bmatrix} \in \mathbb R^{2n \times 6} .
\end{equation*}
Denote the GN refinement of $\hat {\bf s}_n$ and $\hat {\bf t}_n$ by $\hat {\bf s}^{\rm GN}_n$ and $\hat {\bf t}^{\rm GN}_n$, respectively. Given the initial consistent estimate $\hat {\bf R}_n^{\rm BE}$ and $\hat {\bf t}_n^{\rm BE}$, we have 
\begin{equation}  \label{GN_t}
\begin{split}
\begin{bmatrix}
\hat {\bf s}^{\rm GN}_n \\
\hat {\bf t}^{\rm GN}_n
\end{bmatrix}=
\begin{bmatrix}
{\bf 0} \\
\hat {\bf t}_n^{\rm BE}
\end{bmatrix}+ &
\left({\bf J}^\top (\hat {\bf R}_n^{\rm BE}, \hat {\bf t}_n^{\rm BE}){\bf J}(\hat {\bf R}_n^{\rm BE}, \hat {\bf t}_n^{\rm BE}) \right) ^{-1} \\
& {\bf J}^\top (\hat {\bf R}_n^{\rm BE}, \hat {\bf t}_n^{\rm BE}) \left({\bf b}-f(\hat {\bf R}_n^{\rm BE},\hat {\bf t}_n^{\rm BE}) \right) ,
\end{split}
\end{equation}
where $f(\hat{\bf R},\hat{\bf t})=[f_{1}(\hat{\bf R},0,\hat{\bf t})^\top~\cdots~f_{n}(\hat{\bf R},0,\hat{\bf t})^\top]^\top$.
As such, 
\begin{equation} \label{GN_R}
    \hat{{\bf R}}^{\rm GN}_n=\hat{{\bf R}}_n^{\rm BE}\exp\left({{}\hat{{\bf s}}_n^{\rm GN}}^{\wedge} \right).
\end{equation}

We remark here that our devised PnP solver has overall $O(n)$ computational complexity. Specifically, the $O(n)$ calculations in our algorithm include the computation of the centroid of 3D reference points $\bar {\bf p}^w$, the estimation of noise variance, the calculation of the consistent estimate $\hat {\bm \theta}_n^{\rm BE}$, and the constrained GN iterations. The other operations consume constant time. Therefore, our PnP solver is efficient and favorable in large sample applications. The whole algorithm is summarized in Algorithm~\ref{pseudo_algorithm_CPnP}.
\begin{algorithm}
	\caption{ CPnP: Consistent PnP Pose Estimator}
	\label{pseudo_algorithm_CPnP}
	\begin{algorithmic}[1]
		\Statex {\bf Input:} 3D points $({\bf p}_i^w)_{i=1}^{n}$ and 2D projections $({\bf q}_i)_{i=1}^{n}$.
		\Statex {\bf Output:} the estimates of $\bf R$ and $\bf t$.
		\State Calculate $({\bf q}_i')_{i=1}^{n}$ according to~\eqref{projection_equation};
		\State Estimate the variance of projection noises;
		\State Calculate the bias-eliminated solution $\hat {\bm \theta}_n^{\rm BE}$ in~\eqref{bias_eliminated_solution};
		\State Recover $\hat {\bf R}_n^{\rm BE}$ and $\hat {\bf t}_n^{\rm BE}$ according to~\eqref{recover_alpha}-\eqref{recover_t};
		\State Project the rotation matrix into ${\rm SO}(3)$ using~\eqref{SVD_projection};
		\State Refine the estimate using the GN iterations~\eqref{GN_t} and~\eqref{GN_R}.
	\end{algorithmic}
\end{algorithm}

\section{Experiments}

In this section, we compare our algorithm, referred to as CPnP, with some well-known PnP solvers, including EPnP, EPnP+GN~\cite{lepetit2009epnp}, MLPnP, MLPnP+GN~\cite{urban2016mlpnp}, and DLS~\cite{hesch2011direct}. The results will be presented in terms of estimation accuracy and computing time. To fairly make the comparison, we implement all the algorithms in Matlab, and all experiments are conducted via a laptop with Apple M1 Pro.

\begin{figure*}[!t]
\centering
\begin{subfigure}[b]{0.24\textwidth}
\centering
\includegraphics[width=1\textwidth]{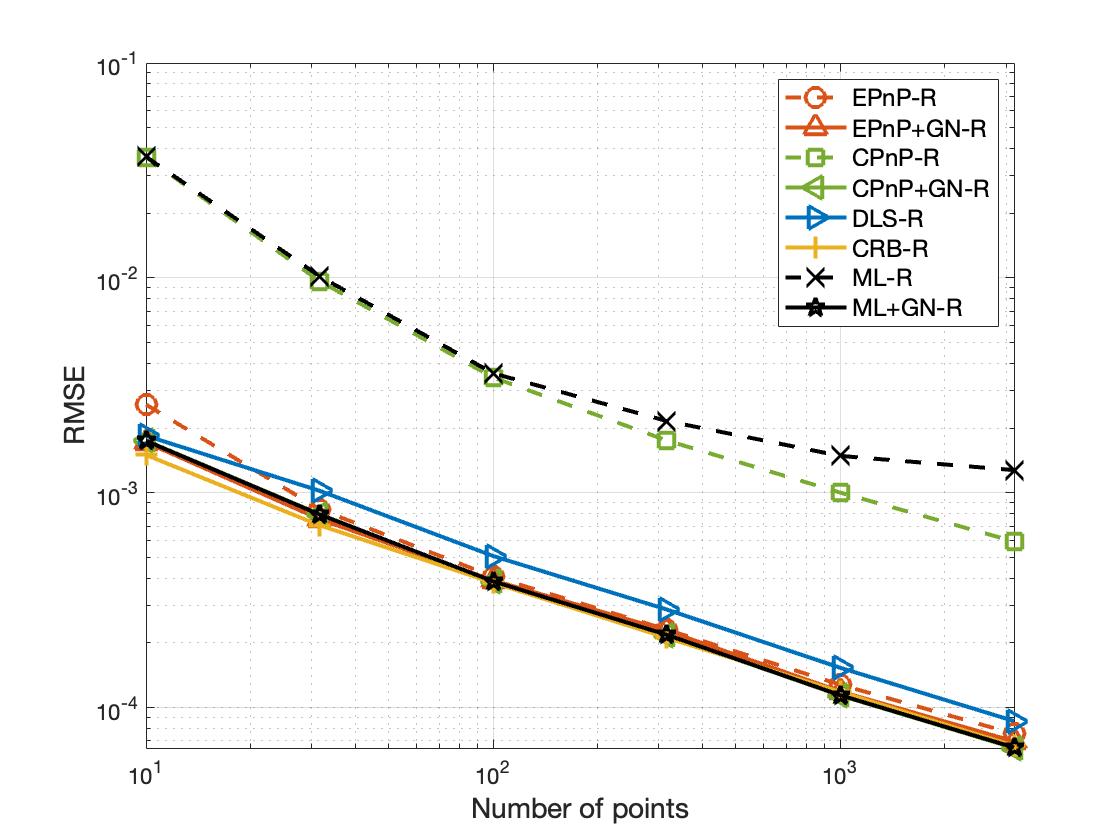}
\caption{$\sigma$=2 pixels ($\bf R$)}
\end{subfigure}
\begin{subfigure}[b]{0.24\textwidth}
\centering
\includegraphics[width=1\textwidth]{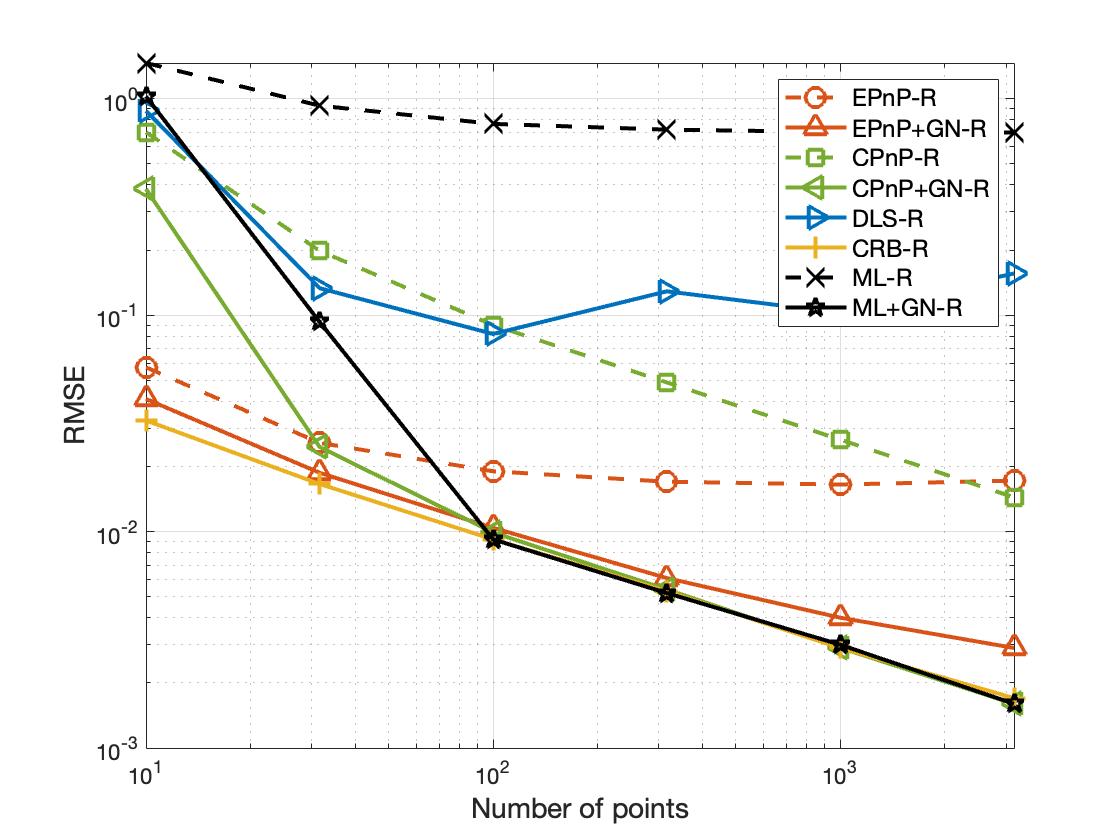}
\caption{$\sigma$=5 pixels ($\bf R$)}
\end{subfigure}
\begin{subfigure}[b]{0.24\textwidth}
\centering
\includegraphics[width=1\textwidth]{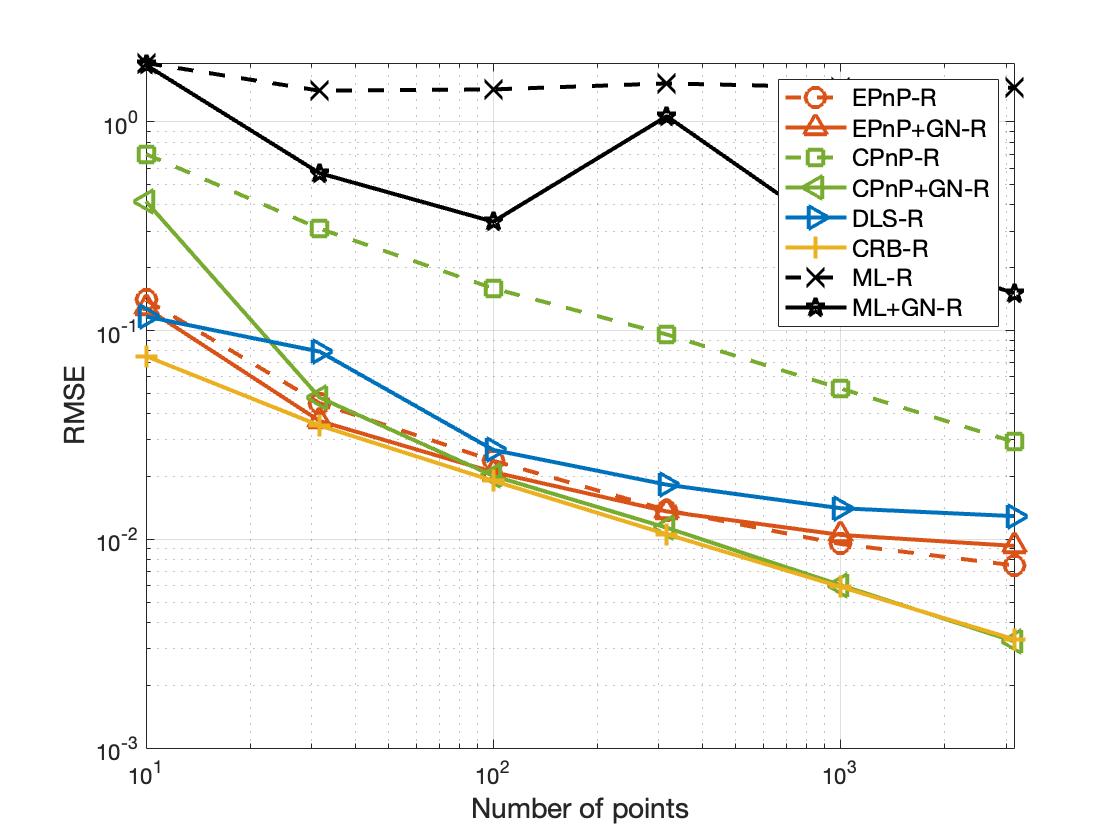}
\caption{$\sigma$=10 pixels ($\bf R$)}
\end{subfigure}
\begin{subfigure}[b]{0.24\textwidth}
\centering
\includegraphics[width=1\textwidth]{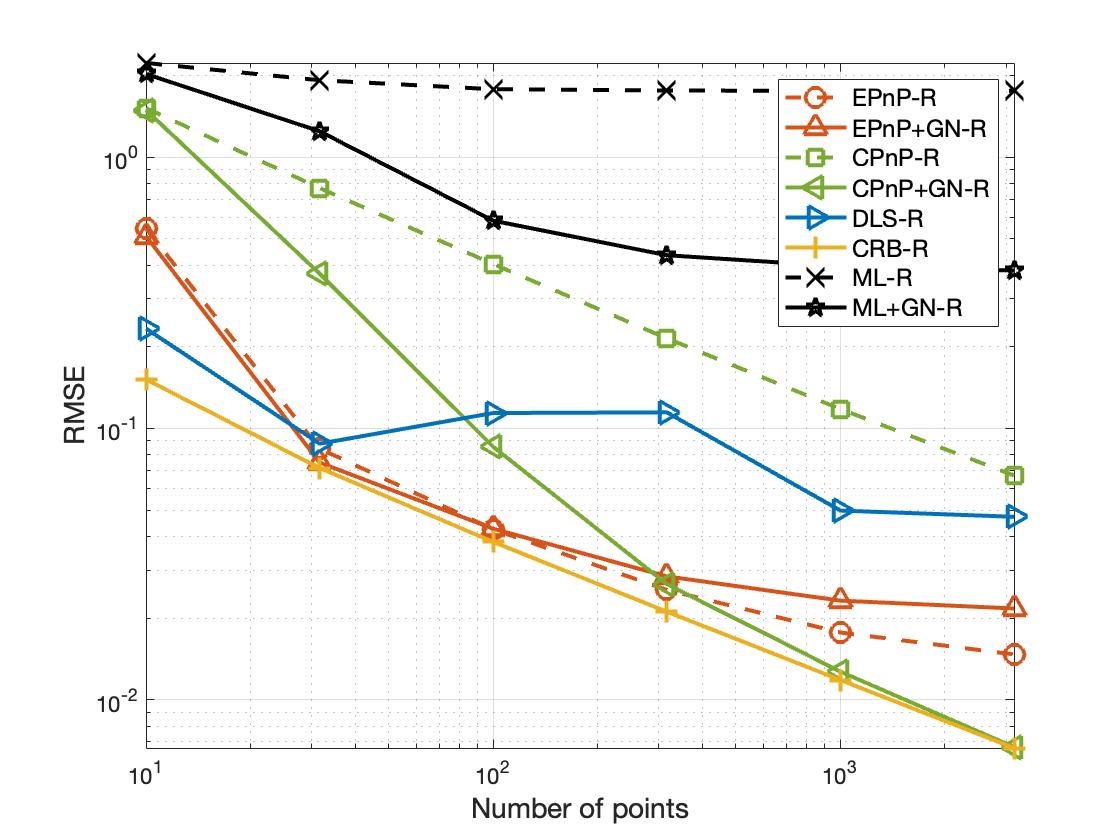}
\caption{$\sigma$=20 pixels ($\bf R$)}
\end{subfigure}
\begin{subfigure}[b]{0.24\textwidth}
\centering
\includegraphics[width=1\textwidth]{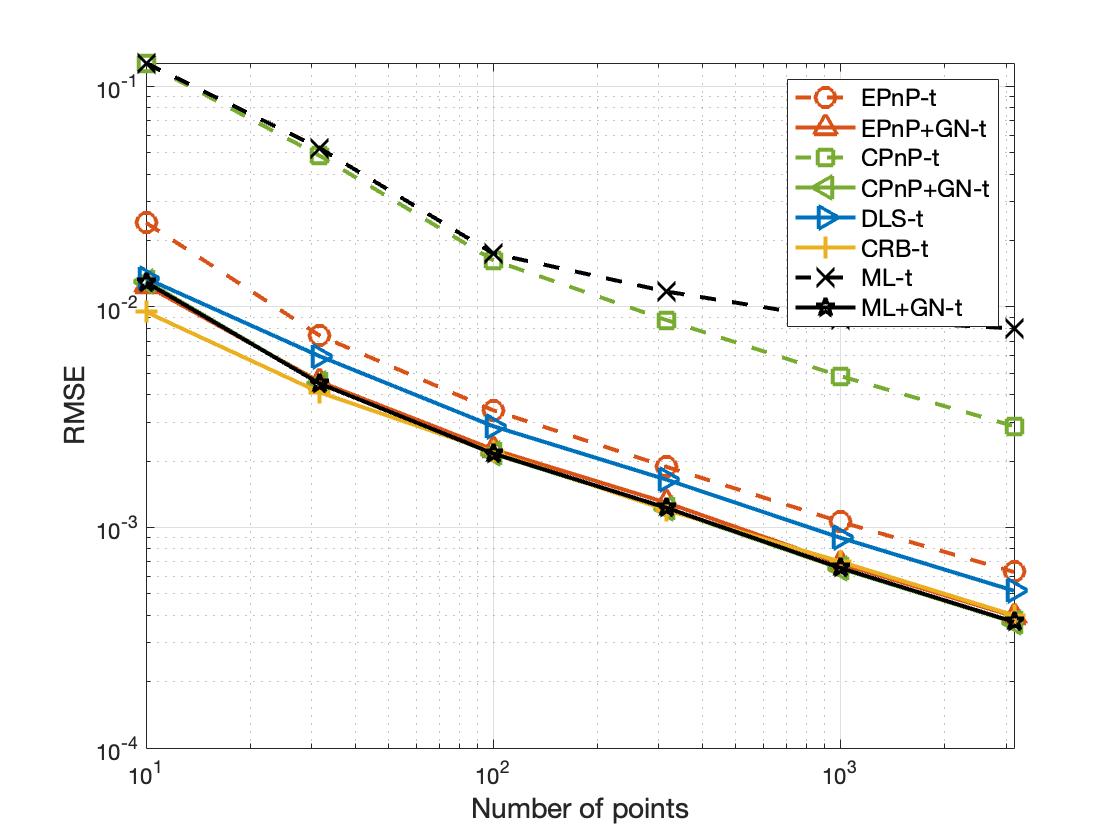}
\caption{$\sigma$=2 pixels ($\bf t$)}
\end{subfigure}
\begin{subfigure}[b]{0.24\textwidth}
\centering
\includegraphics[width=1\textwidth]{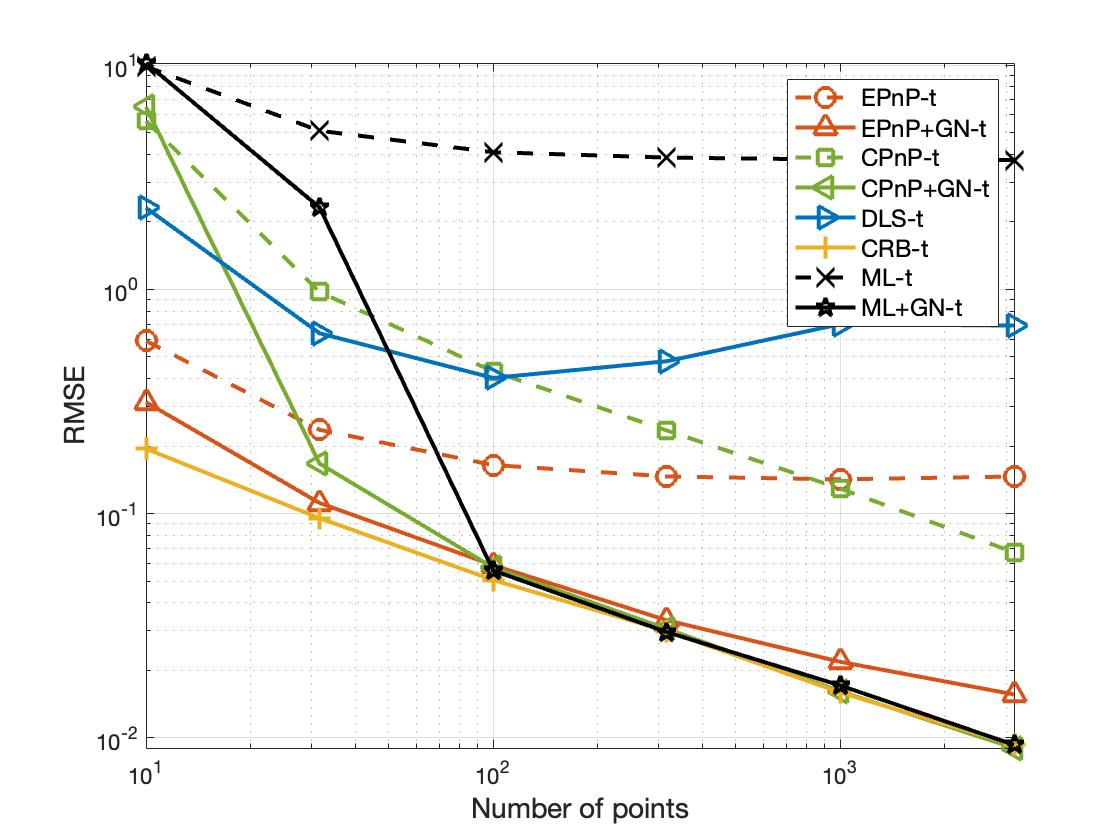}
\caption{$\sigma$=5 pixels ($\bf t$)}
\end{subfigure}
\begin{subfigure}[b]{0.24\textwidth}
\centering
\includegraphics[width=1\textwidth]{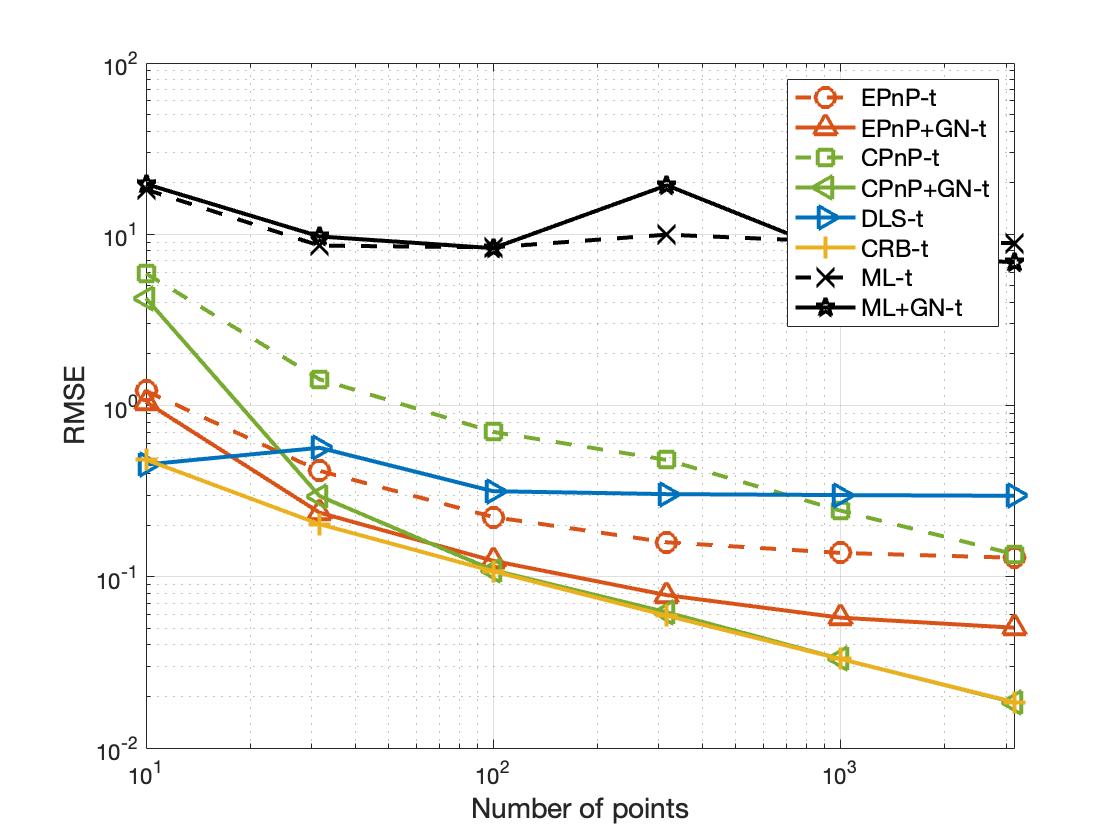}
\caption{$\sigma$=10 pixels ($\bf t$)}
\end{subfigure}
\begin{subfigure}[b]{0.24\textwidth}
\centering
\includegraphics[width=1\textwidth]{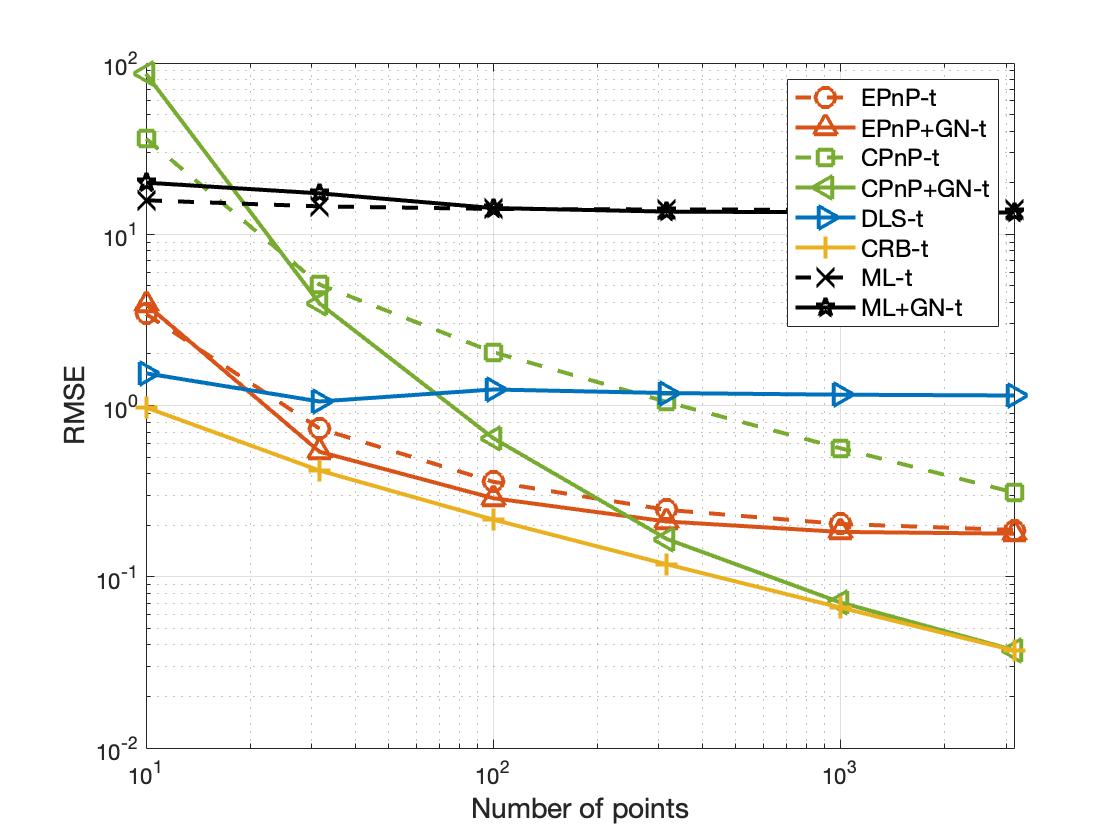}
\caption{$\sigma$=20 pixels ($\bf t$)}
\end{subfigure}
\caption{RMSE comparison among different PnP solvers with synthetic data.}
\label{RMSE_comparison_simulation}
\end{figure*}

\subsection{Experiments with Synthetic Data}
In simulations, the Euler angles of the camera are set as $[\pi/3~\pi/3~\pi/3]^\top$, and the translation vector is $[2~6~6]^\top$. For the intrinsic parameters, the focal length is set as $f_x=f_y=50$mm ($800$ pixels), and the size of the image plane is $640 \times 480$ pixels. The principle point lies in the top-left corner of the image plane and the principle point offset is $u_0=320$ pixels, $v_0=240$ pixels. For the 3D points under the camera frame, we randomly generate them from the region $[-2, 2] \times [-2, 2] \times [4, 16]$ m. After filtering out the points outside the range of the image plane, the remaining 3D points are projected onto the image plane by the projection equation~\eqref{noise_free_projection}. Specifically, the projection noise ${\bm \epsilon}_i$ is Gaussian noise whose standard deviation is $\sigma$ pixels. For each fixed $\sigma$ and number of points $n$, we execute $T=1000$ Monte Carlo tests to evaluate the following root-mean-square errors (RMSE) of each PnP solver:
\begin{align*}
    {\rm RMSE}_{\bf R} & = \sqrt{\frac{1}{T}\sum\limits_{t=1}^{T} {\|\hat {\bf R}(\omega_t)-{\bf R}^o\|}_{\rm F}^2}, \\
    {\rm RMSE}_{\bf t} & = \sqrt{\frac{1}{T}\sum\limits_{t=1}^{T} {\|\hat {\bf t}(\omega_t)-{\bf t}^o\|}^2}.
\end{align*}

\begin{figure*}[!t]
\centering
\begin{subfigure}[t]{0.24\textwidth}
\centering
\includegraphics[width=1\textwidth]{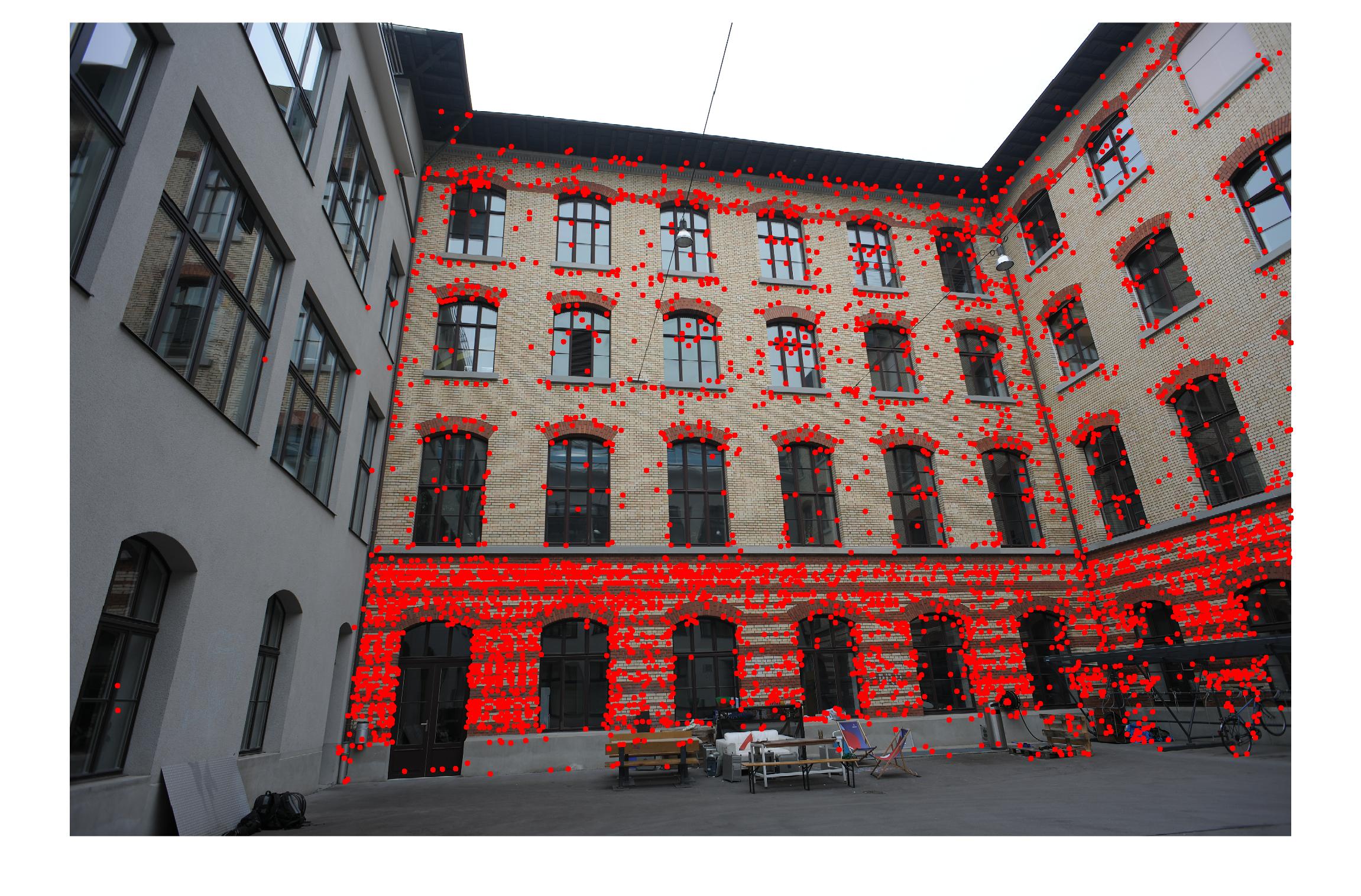}
\caption{Outdoor: Courtyard}
\end{subfigure}
\begin{subfigure}[t]{0.24\textwidth}
\centering
\includegraphics[width=1\textwidth]{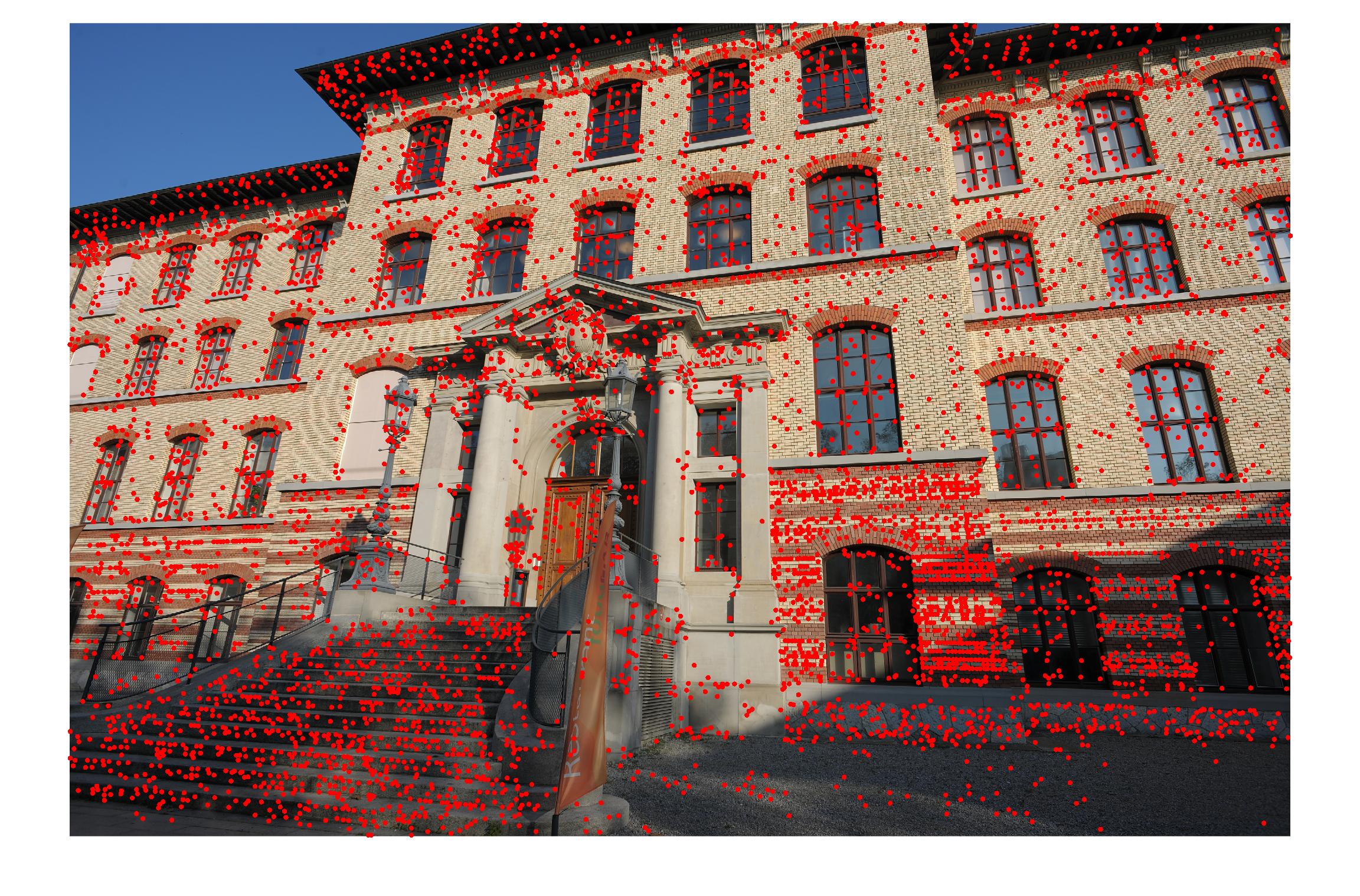}
\caption{Outdoor: Facade}
\end{subfigure}
\begin{subfigure}[t]{0.24\textwidth}
\centering
\includegraphics[width=1\textwidth]{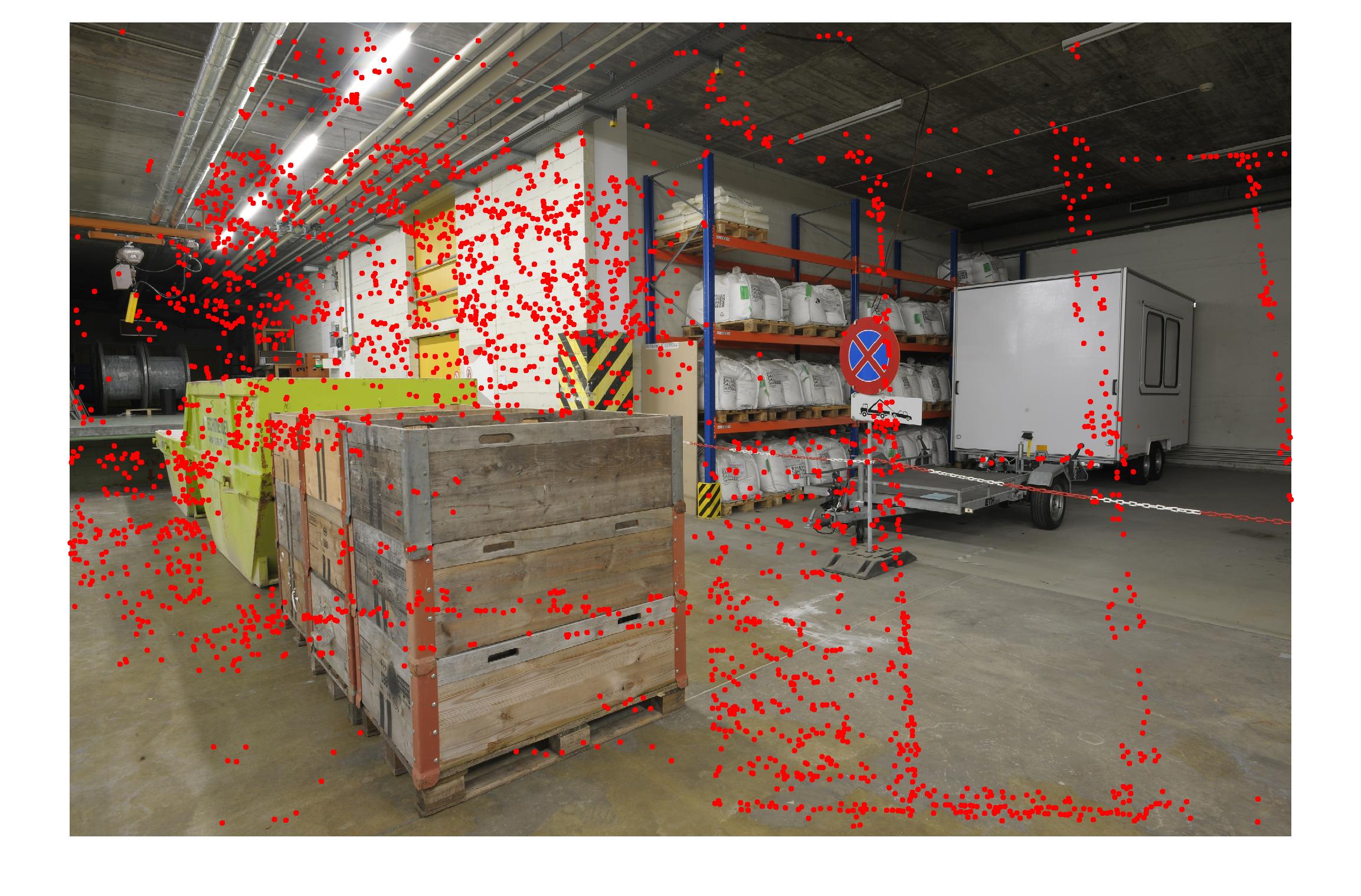}
\caption{Indoor: Delivery Area}
\end{subfigure}
\begin{subfigure}[t]{0.24\textwidth}
\centering
\includegraphics[width=1\textwidth]{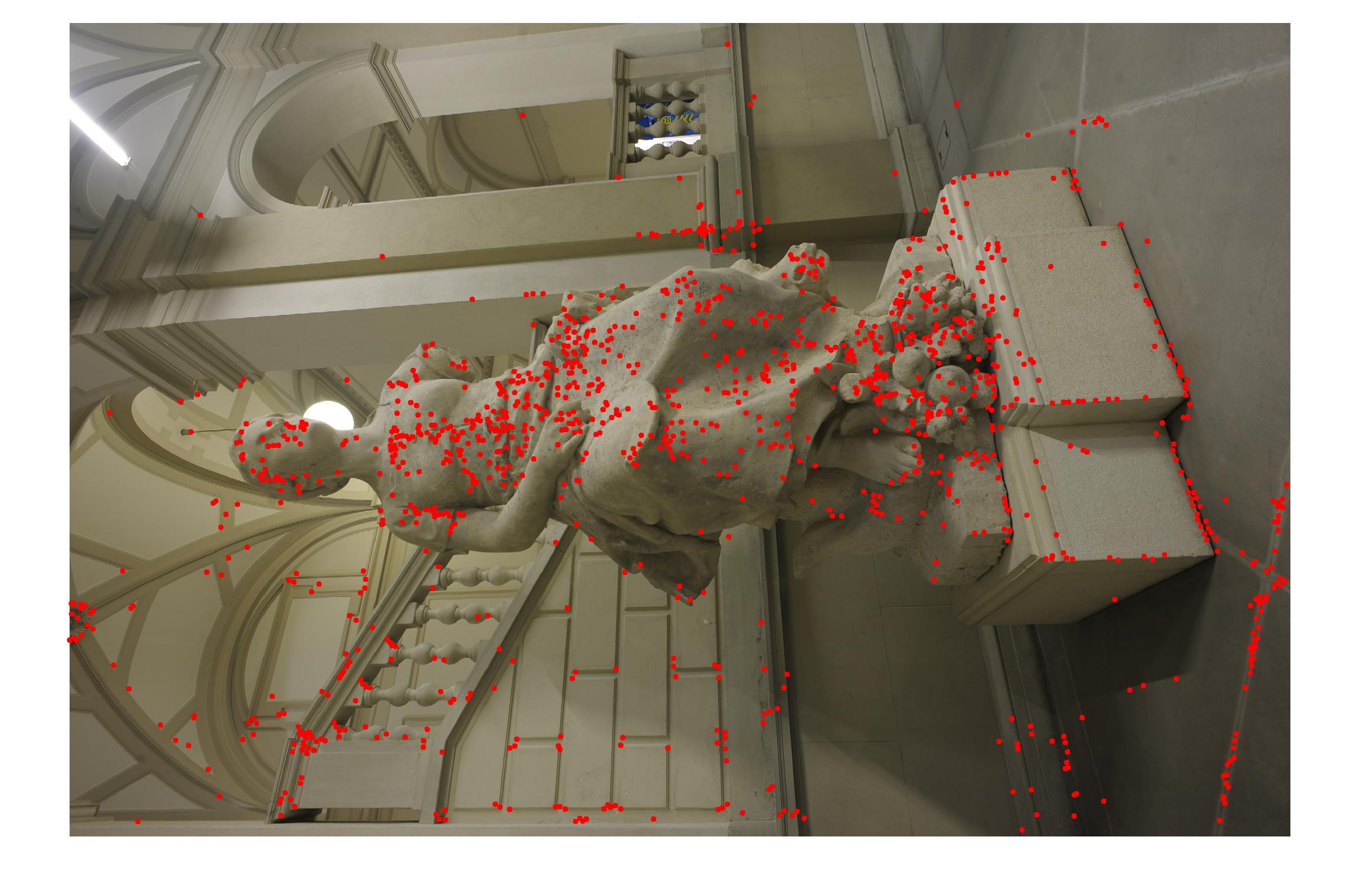}
\caption{Indoor: Statue}
\end{subfigure}
\caption{Several images in ETH3D Benchmark~\cite{schops2017multi}, among which two are outdoor scenarios and two are indoor scenarios.}
\label{ETH3D_figures}
\end{figure*}

\begin{figure*}[!t]
\centering
\begin{subfigure}[b]{0.24\textwidth}
\centering
\includegraphics[width=1\textwidth]{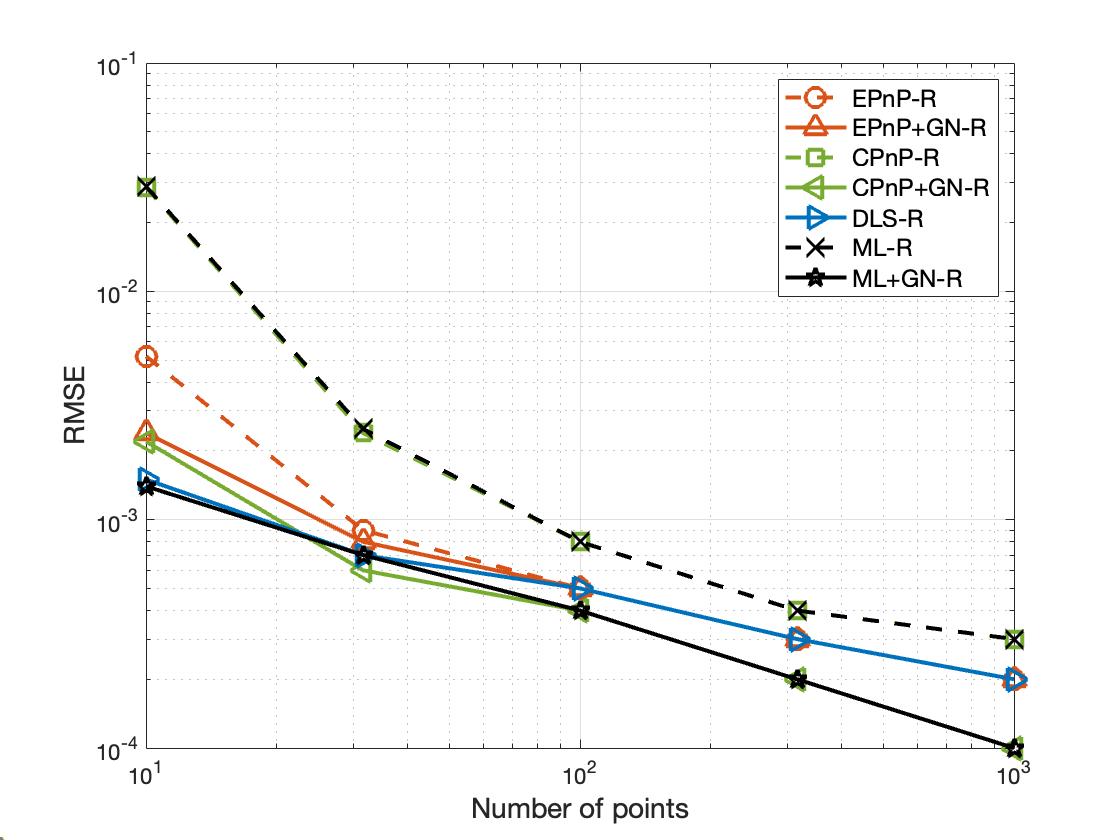}
\caption{Courtyard ($\bf R$)}
\end{subfigure}
\begin{subfigure}[b]{0.24\textwidth}
\centering
\includegraphics[width=1\textwidth]{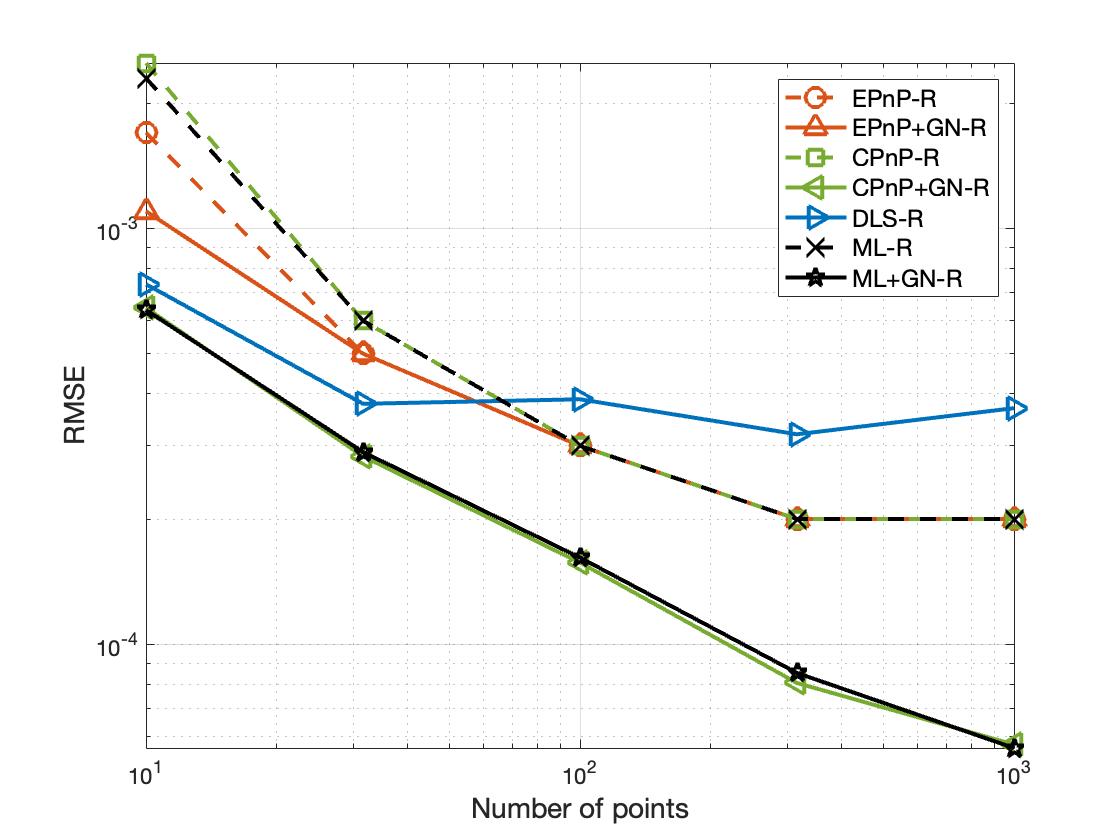}
\caption{Facade ($\bf R$)}
\end{subfigure}
\begin{subfigure}[b]{0.24\textwidth}
\centering
\includegraphics[width=1\textwidth]{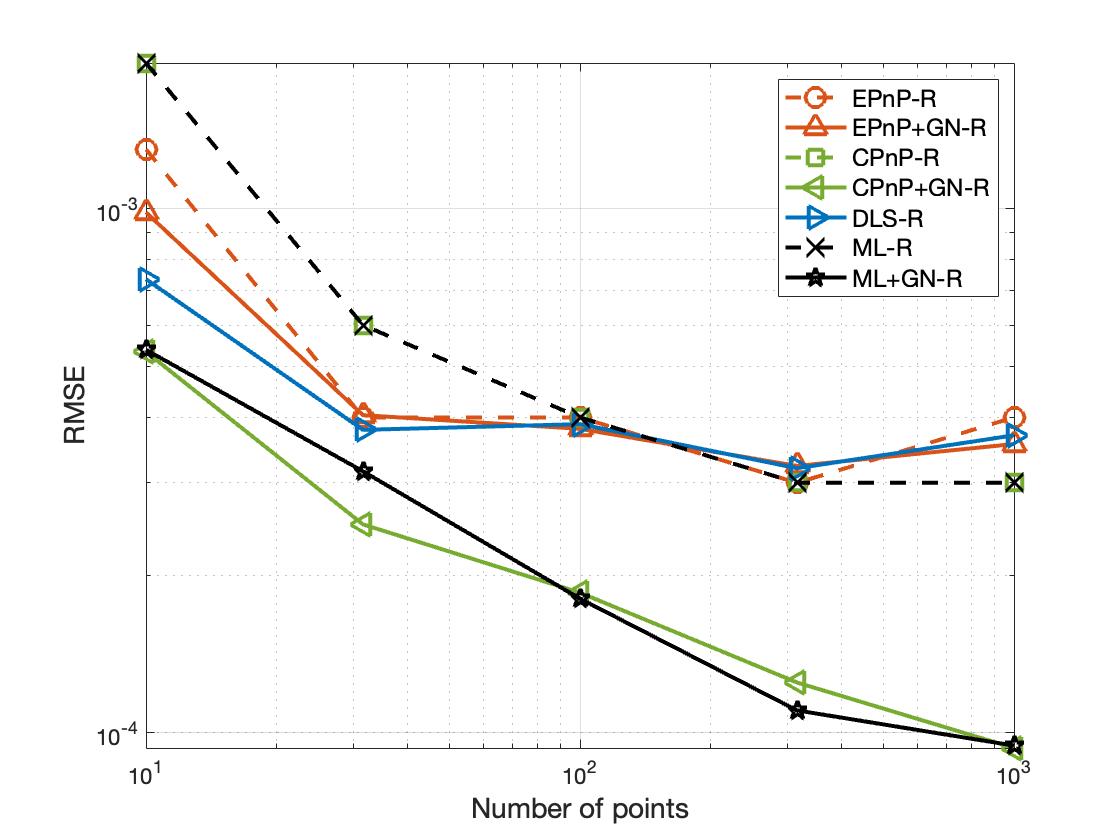}
\caption{Delivery Area ($\bf R$)}
\end{subfigure}
\begin{subfigure}[b]{0.24\textwidth}
\centering
\includegraphics[width=1\textwidth]{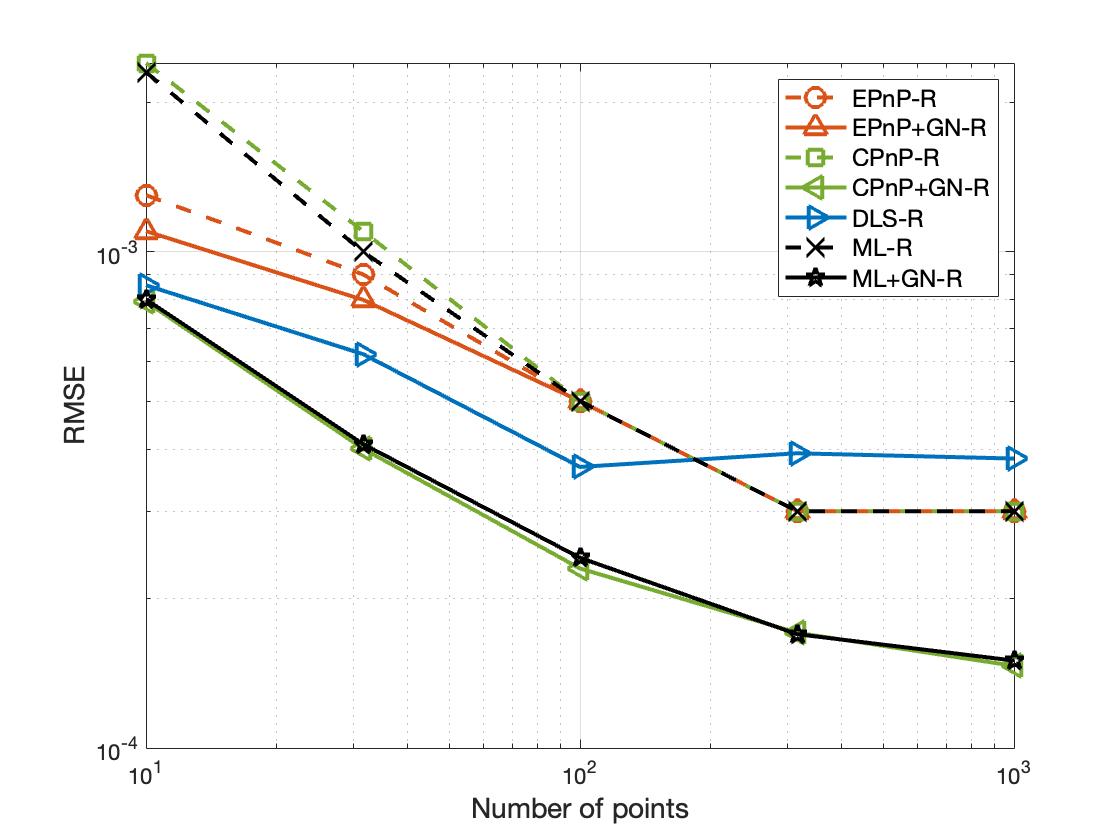}
\caption{Statue ($\bf R$)}
\end{subfigure}
\begin{subfigure}[b]{0.24\textwidth}
\centering
\includegraphics[width=1\textwidth]{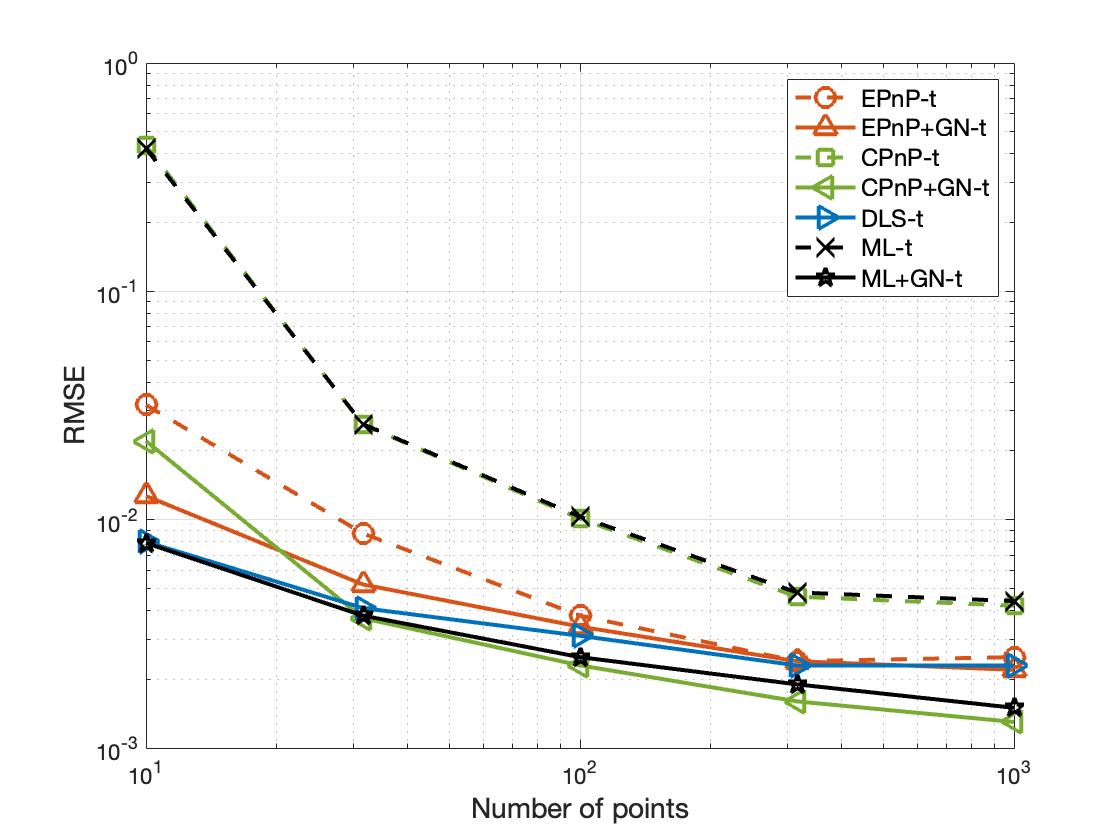}
\caption{Courtyard ($\bf t$)}
\end{subfigure}
\begin{subfigure}[b]{0.24\textwidth}
\centering
\includegraphics[width=1\textwidth]{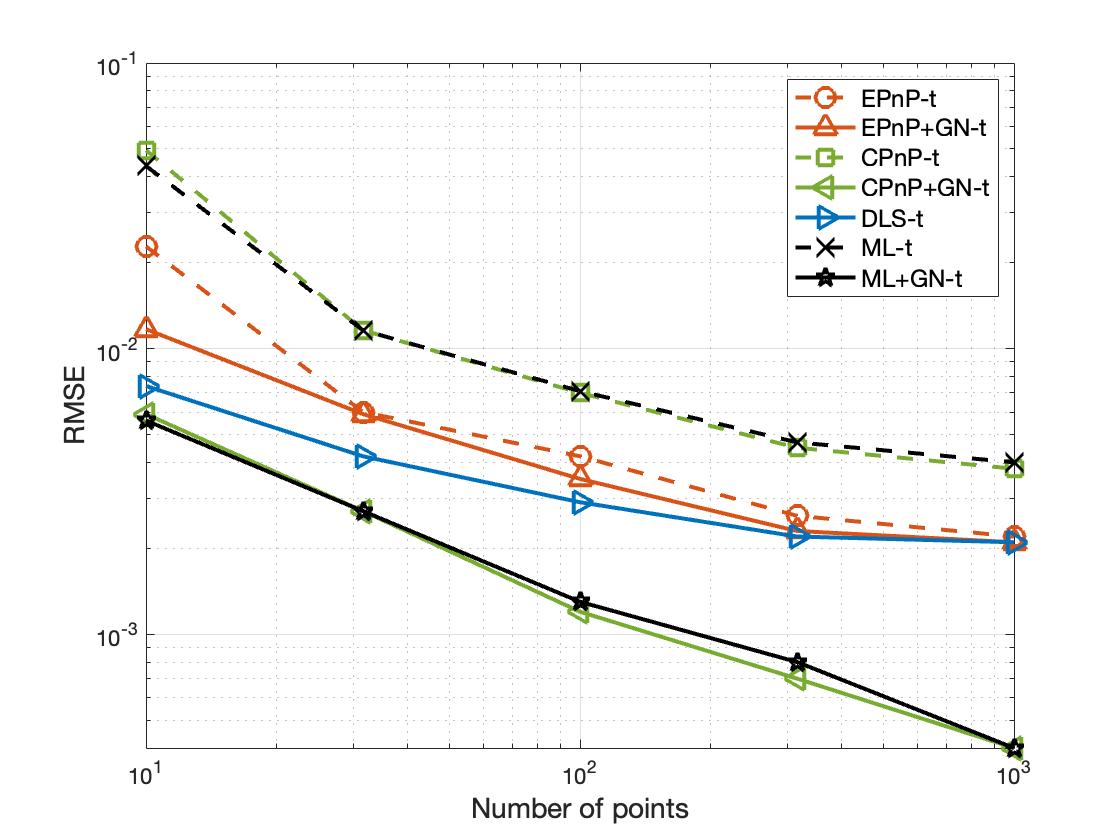}
\caption{Facade ($\bf t$)}
\end{subfigure}
\begin{subfigure}[b]{0.24\textwidth}
\centering
\includegraphics[width=1\textwidth]{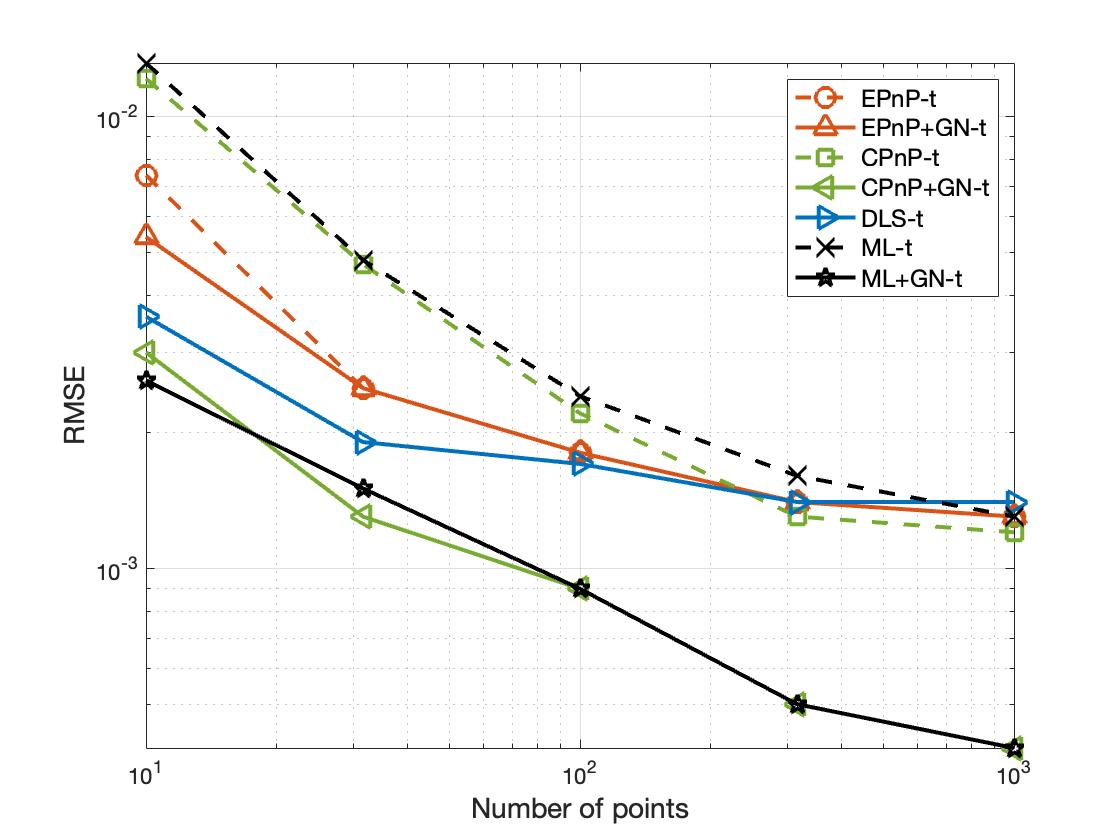}
\caption{Delivery Area ($\bf t$)}
\end{subfigure}
\begin{subfigure}[b]{0.24\textwidth}
\centering
\includegraphics[width=1\textwidth]{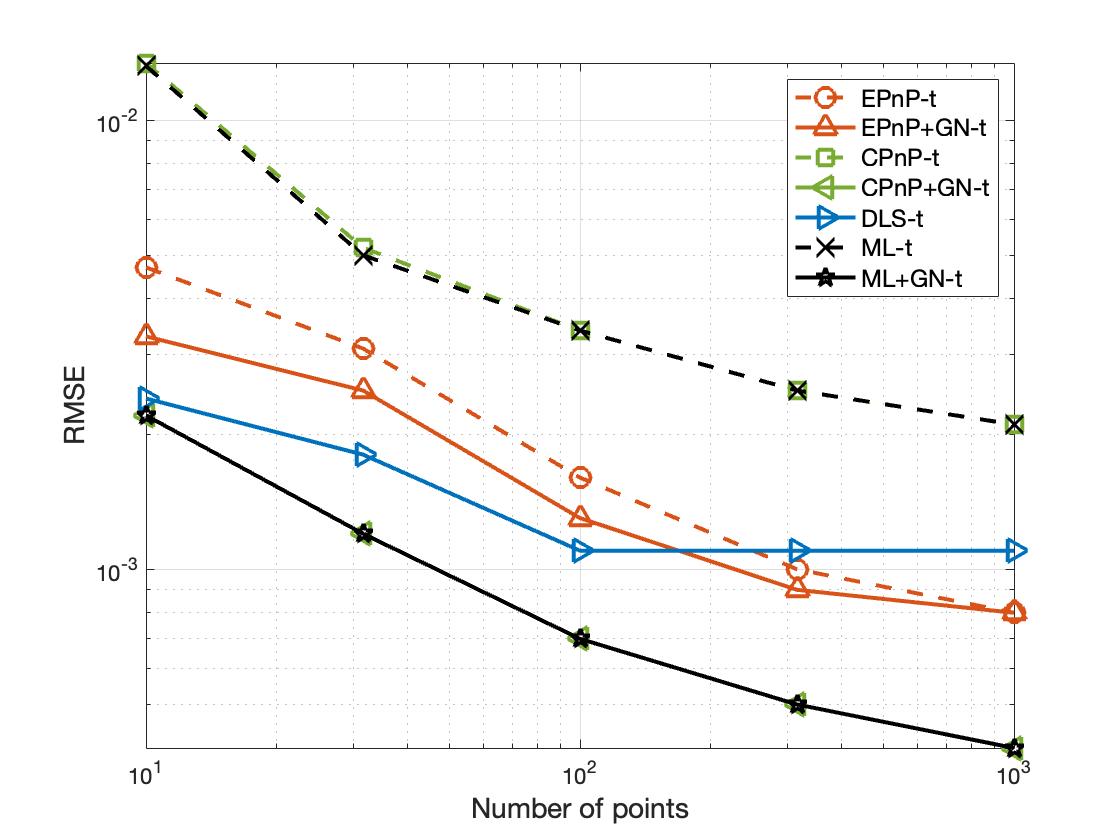}
\caption{Statue ($\bf t$)}
\end{subfigure}
\begin{subfigure}[b]{0.24\textwidth}
\centering
\includegraphics[width=1\textwidth]{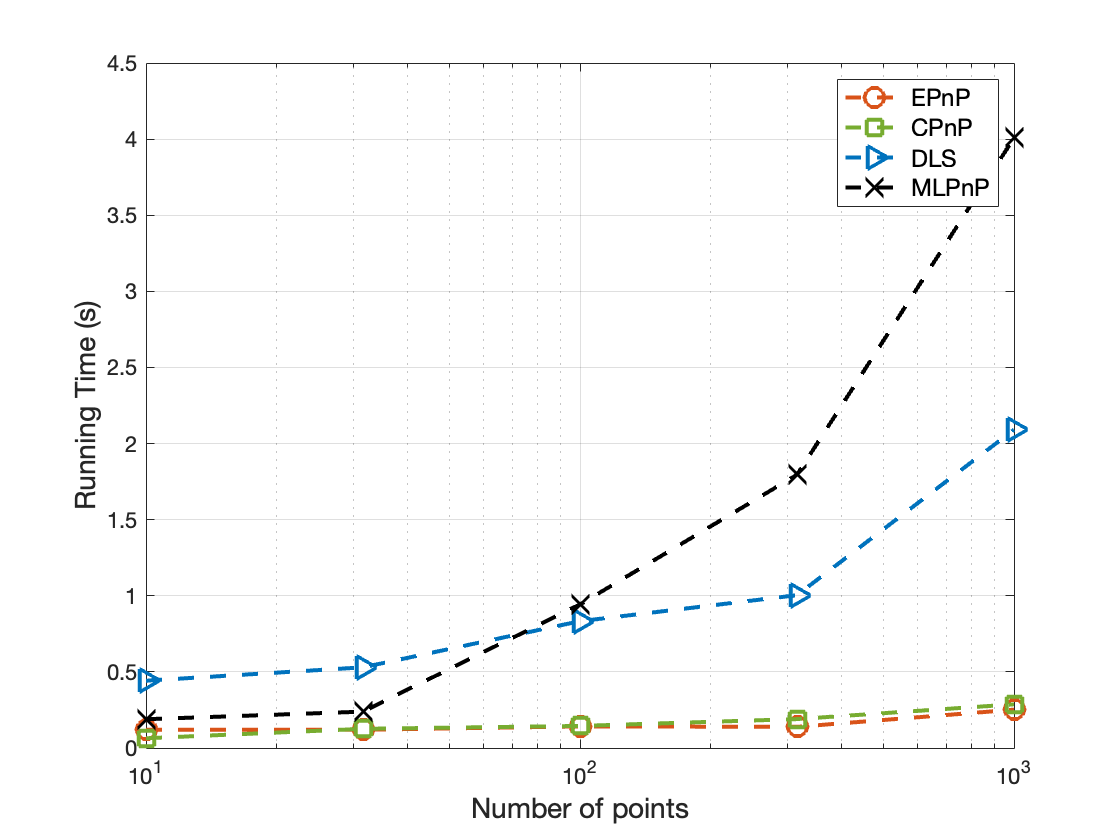}
\caption{Courtyard (CPU time)}
\end{subfigure}
\begin{subfigure}[b]{0.24\textwidth}
\centering
\includegraphics[width=1\textwidth]{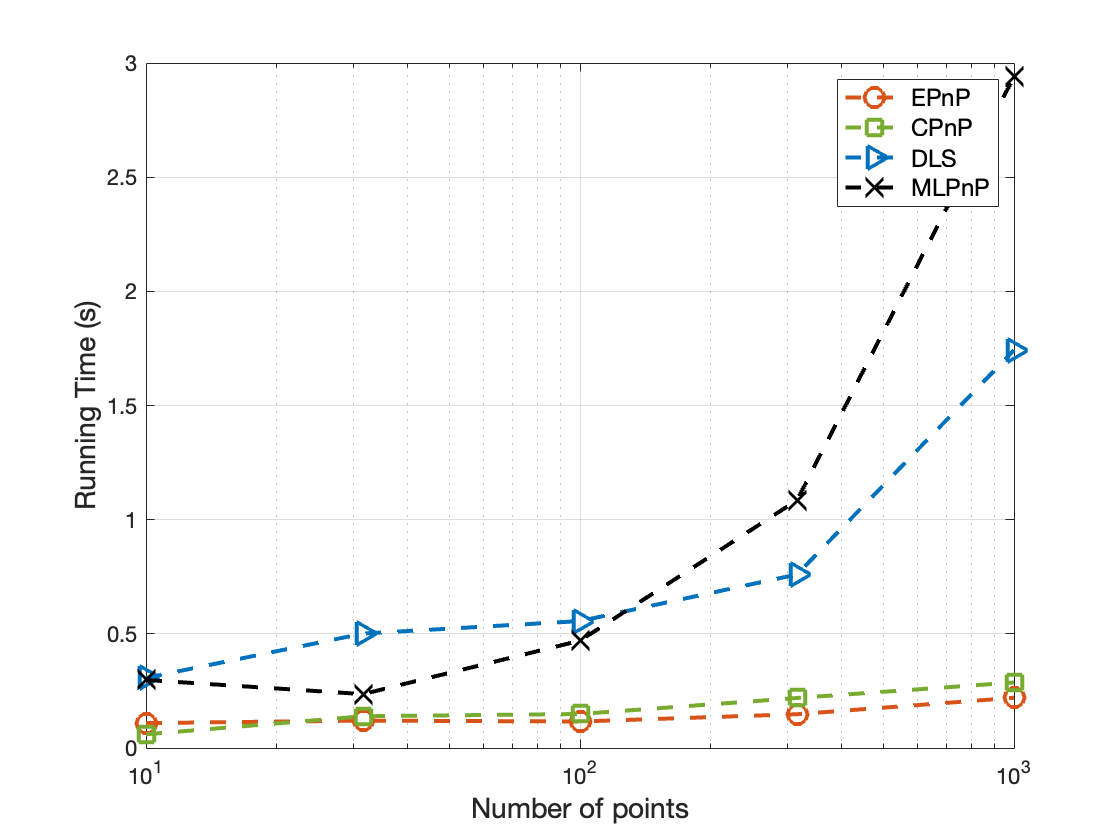}
\caption{Facade (CPU time)}
\end{subfigure}
\begin{subfigure}[b]{0.24\textwidth}
\centering
\includegraphics[width=1\textwidth]{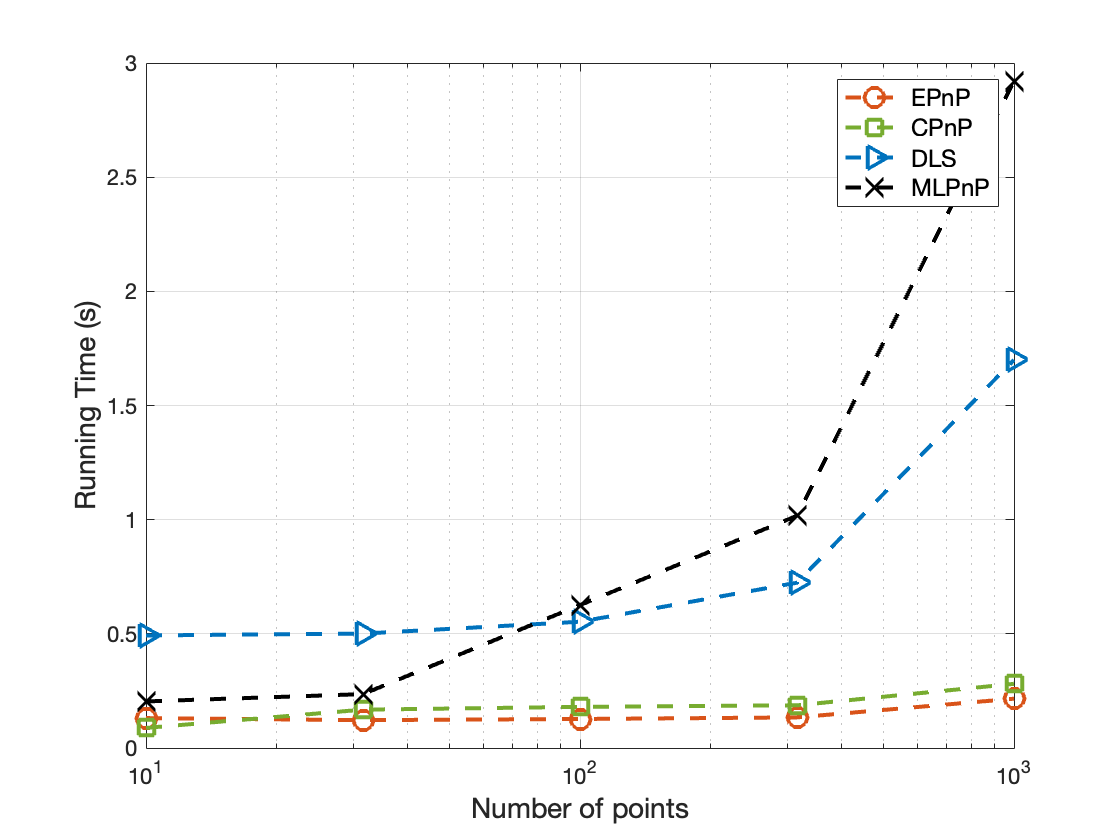}
\caption{Delivery Area (CPU time)}
\end{subfigure}
\begin{subfigure}[b]{0.24\textwidth}
\centering
\includegraphics[width=1\textwidth]{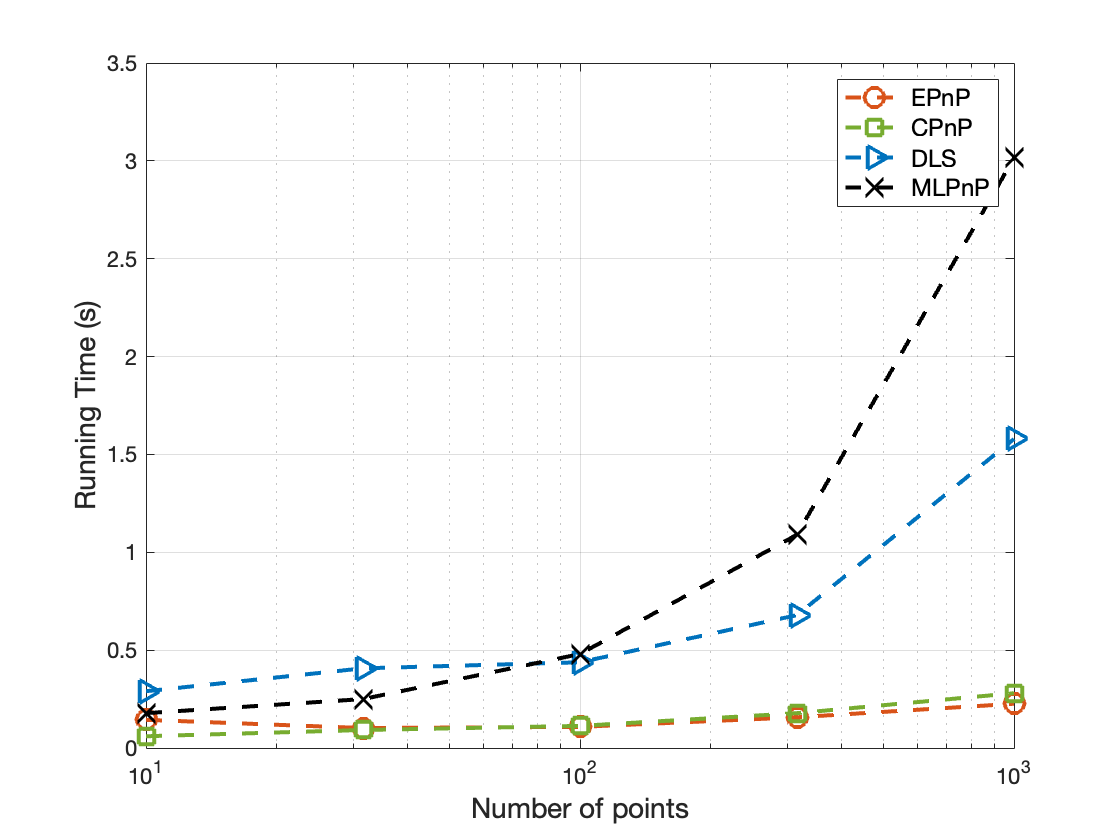}
\caption{Statue (CPU time)}
\end{subfigure}
\caption{RMSE and CPU time comparisons among different PnP solvers with real images.}
\label{Comparison_real_images}
\end{figure*}

We set the intensities of projection noises as $\sigma=2,5,10,20$ pixels, respectively, and for each choice of $\sigma$, the number of points varies from $10$ to $3000$. The RMSE comparison among different PnP solvers is presented in Fig.~\ref{RMSE_comparison_simulation} where the logarithmic axes are used. ``CRB'' in the figure offers the theoretical lower bound, named Cramer-Rao bound, for the RMSE of any unbiased estimator. For the derivation of the constrained CRB, one can refer to~\cite{stoica1998cramer}. We can see from the figure that when the intensity of noises is small, all estimators with GN iterations can achieve the CRB as the number of points increases. This is because in that case, the biases of the other estimators are negligible, and the GN iterations can make the estimate converge to the global minimizer of the original ML problem. Nevertheless, with the increase of noise intensity, all the other PnP solvers have relatively large biases. As a result, the initial estimate may not locate within the attraction region of the global minimum, and the GN iterations may only obtain local minima. However, since our proposed estimator has undergone bias elimination, it is consistent---as the number of points goes large, the estimate can converge to the true camera pose. With the initial consistent estimate, the GN iterations can further obtain the global minimum of the ML problem~\eqref{LS_problem}, achieving the theoretical lower bound. It can be seen that when the number of points exceeds several hundred, our estimator has minimal RMSEs.

\subsection{Experiments with Real Images}

We use the real-world dataset called ETH3D Benchmark~\cite{schops2017multi} to test the performances of our developed algorithm. This dataset provides original images, the intrinsic parameters of the cameras, the coordinates of 3D points and the corresponding 2D projections, and the true pose of cameras associated with each image. 
As shown in Fig.~\ref{ETH3D_figures}, the evaluation is done in four images---two are outdoor scenarios, and two are indoor scenarios. For each image, the detected 2D points will first be cleaned to make sure each 2D point can match a corresponding 3D point. We then randomly select a certain number of points from the cleaned dataset to estimate the camera pose. For each number of points, totally $T=50$ random selections are conducted to calculate the RMSE.

We vary the number of points from $10$ to $1000$. The RMSE and CPU time comparisons are shown in Fig.~\ref{Comparison_real_images}. The RMSE of our estimator is similar to that of MLPnP, which is superior to the other estimators, especially when the number of points is large. The trend of the lines of CPnP shows the consistent property of our proposed algorithm. 
It is worth noting that since the points may not be sufficient enough to support $T=50$ Monte Carlo repetitions, the slope has slowed down a little when the selected number of points is $1000$.
Regarding the comparison of computational complexity, we record the CPU time of each algorithm in $50$ Monte Carlo runs. The CPU time of different estimators is comparable when the number of points is below 100. They deviate when the number further increased, with MLPnP increasing most sharply and DLS the second. Our proposed CPnP remains remarkably stable and maintains low CPU time as the EPnP does. In the case of $1000$ points, the proposed CPnP consumes less than one second in total $50$ estimations, showing the applicability in real-time applications. 

\section{Conclusion}
In this paper, we revisited the PnP problem from the view of statistics. Variable elimination was conducted to overcome the scale ambiguity and homogeneity of the original projection model. On the basis of the consistent estimate of noise variance, a consistent estimator for the camera pose was proposed via a bias-eliminated closed-form solution. Constrained GN iterations were executed to refine the initial estimate. Our proposed CPnP estimator has two advantages: it is consistent; its computational complexity is $O(n)$. Experiments using both synthetic data and benchmark datasets showed that the proposed CPnP is superior to the other estimators for images with dense visual features, in terms of RMSE and CPU time.

\appendices
\section{Consistent estimation of noise variance} \label{noise_variance_estimation}
We adopt the method in~\cite{mu2017globally} to estimate $\sigma^2$. The work in~\cite{mu2017globally} is about the identification of a rational system where the estimate of the noise variance is obtained by solving a generalized eigenvalue problem. Next, we will briefly outline the estimation procedure.
Define two matrices by the available data and information
\begin{align}
{\bm \Phi} &=\frac{1}{n} \begin{bmatrix}
{\bf A}^\top {\bf A} ~~ {\bf A}^\top {\bf b} \\
{\bf b}^\top {\bf A} ~~ {\bf b}^\top {\bf b}
\end{bmatrix}, \\
{\bm \Delta} &= \frac{1}{n} \begin{bmatrix}
{\bf G}^\top {\bf G} ~~~ {\bf G}^\top {\bf 1}_{2n \times 1} \\
{\bf 1}_{2n \times 1}^\top {\bf G} ~~~~~~ 2n
\end{bmatrix}.
\end{align}
Based on ${\bm \Phi}$ and ${\bm \Delta}$, define a function over the variable $\lambda$ as ${\bf H}(\lambda)={\bm \Phi}-\lambda {\bm \Delta}$. Then, the function defined as $h(\lambda)={\rm det}({\bf H}(\lambda))$ is a polynomial of power $4$ over $\lambda$. As a result, $h(\lambda)=0$ has $4$ roots and denote all the roots by $\{\lambda_1,\ldots,\lambda_4\}$. Finally, the estimate of noise variance is given by
\begin{equation}\label{noise_estimate}
\hat \sigma^2 = {\rm min} \{\lambda_1,\ldots,\lambda_4\}. 
\end{equation}
Given Assumptions~\ref{noise_assump}-\ref{reference_point_assump}, based on Lemma 2 in~\cite{mu2017globally}, it holds that $\sqrt{n} (\hat \sigma^2-\sigma^2)=O_p(1)$, i.e., $\hat \sigma^2$ is a $\sqrt{n}$-consistent estimate of $\sigma^2$. 

\section{Consistency of $\hat {\bm \theta}_n^{\rm BE}$} \label{consistency_of_first_step}
Rewrite~\eqref{variable_eliminate_equation} as follows:
\begin{equation}\label{noie_free_equation}
{\bf q}_i'=\alpha {\bf W}{\bf E} ({\bf R} {\bf p}_i^w + {\bf t}) - \alpha \left({\bf p}_i^w- \bar {\bf p}^w \right) ^\top {\bf r}_3 {\bf q}_i^o+ {\bm \epsilon}_i,
\end{equation}
where ${\bf q}_i^o=[u_i^o~v_i^o]^\top$ is the noise-free counterpart of ${\bf q}_i'$. Concatenating~\eqref{noie_free_equation} for all reference points yields the matrix from:
\begin{equation}\label{noise_free_matrix_form}
{\bf b}= {\bf B} {\bm \theta} + {\bm \epsilon},
\end{equation}
where 
\begin{equation*}
{\bf B}=\begin{bmatrix}
-u_1^o \left({\bf p}_1^w- \bar {\bf p}^w \right) ^\top~f_x{{\bf p}_1^w}^\top~f_x~{\bf 0}_{4 \times 1} \\
-v_1^o \left({\bf p}_1^w- \bar {\bf p}^w \right) ^\top~{\bf 0}_{4 \times 1}~f_y{{\bf p}_1^w}^\top~f_y \\
\vdots \\
-u_n^o \left({\bf p}_n^w- \bar {\bf p}^w \right) ^\top~f_x{{\bf p}_n^w}^\top~f_x~{\bf 0}_{4 \times 1} \\
-v_n^o \left({\bf p}_n^w- \bar {\bf p}^w \right) ^\top~{\bf 0}_{4 \times 1}~f_y{{\bf p}_n^w}^\top~f_y
\end{bmatrix},~~ \\
{\bm \epsilon}=\begin{bmatrix}
{\bm \epsilon}_1 \\
\vdots \\
{\bm \epsilon}_n
\end{bmatrix}.
\end{equation*}
In general, the matrix $\bf B$ has full column rank. As a result, a closed-form solution from~\eqref{noise_free_matrix_form} is given by
\begin{equation}\label{unbiased_LS_solution}
\hat {\bm \theta}_n^{\rm UB}=\left( {\bf B}^\top {\bf B}\right) ^{-1} \left( {\bf B}^\top {\bf b}\right) .
\end{equation}
\begin{lemma} \label{consistency_of_noise_free_case}
The estimate $\hat {\bm \theta}_n^{\rm UB}$ is consistent. 
\end{lemma}
\begin{proof}
The proof is based on the following lemma:
\begin{lemma}[{\cite[Lemma 4]{mu2017globally}}] 
\label{property_of_bounded_variance}
	Let $\{X_k\}$ be a sequence of independent random variables with $\mathbb E[X_k]=0$ and $\mathbb E\left[{X_k}^2 \right]  \leq \varphi <\infty$ for all $k$. Then, there holds $\sum_{k=1}^{n}X_k/\sqrt{n}=O_p(1)$.
\end{lemma}
As long as the sample distribution of $({\bf p}_i^w)_{i=1}^n$ converges to some distribution and Assumption~\ref{reference_point_assump} holds, $\frac{{\bf B}^\top {\bf B}}{n}$ converges to a constant matrix. By further using Lemma~\ref{property_of_bounded_variance}, we obtain
\begin{align*}
    \hat {\bm \theta}_n^{\rm UB} &=\left( \frac{1}{n}{\bf B}^\top {\bf B}\right) ^{-1} \left( \frac{1}{n}{\bf B}^\top {\bf b}\right) \\
    & = \left( \frac{1}{n}{\bf B}^\top {\bf B}\right) ^{-1} \left( \frac{1}{n}{\bf B}^\top {\bf B} {\bm \theta}^o + \frac{1}{n}{\bf B}^\top {\bm \epsilon} \right) \\
    & = {\bm \theta}^o + O_p(1/\sqrt{n}),
\end{align*}
which completes the proof. 
\end{proof}
However, ${\bf q}_i^o$ is unknown, and thus $\hat {\bm \theta}_n^{\rm UB}$ is unavailable. Note that the available information is $\bf A$ and $\bf b$. The main idea of bias elimination is to analyze the gap between $\frac{{\bf A}^\top {\bf A}}{n}$ and $\frac{{\bf B}^\top {\bf B}}{n}$ and that between $\frac{{\bf A}^\top {\bf b}}{n}$ and $\frac{{\bf B}^\top {\bf b}}{n}$. By subtracting the gaps, we can eliminate the bias of $\hat {\bm \theta}_n^{\rm B}$ and achieving the solution $\hat {\bm \theta}_n^{\rm UB}$ asymptotically. Let
\begin{equation*}
{\bf F}={\bf A}-{\bf B}=\begin{bmatrix}
-\epsilon_{11} \left({\bf p}_1^w- \bar {\bf p}^w \right) ^\top~{\bf 0}_{1 \times 8} \\
-\epsilon_{12} \left({\bf p}_1^w- \bar {\bf p}^w \right) ^\top~{\bf 0}_{1 \times 8} \\
\vdots \\
-\epsilon_{n1} \left({\bf p}_n^w- \bar {\bf p}^w \right) ^\top~{\bf 0}_{1 \times 8} \\
-\epsilon_{n2} \left({\bf p}_n^w- \bar {\bf p}^w \right) ^\top~{\bf 0}_{1 \times 8}
\end{bmatrix}.
\end{equation*}
Note that 
\begin{align*}
\frac{1}{n} {\bf A}^\top {\bf A}- \frac{1}{n} {\bf B}^\top {\bf B} &= \frac{1}{n} \left( {\bf B}^\top {\bf F} +{\bf F}^\top {\bf B}+{\bf F}^\top {\bf F}\right) \\
&= \frac{1}{n} {\bf F}^\top {\bf F} +O_p\left( \frac{1}{\sqrt{n}}\right) \\
&=\sigma^2 \frac{1}{n}  {\bf G}^\top {\bf G} +O_p\left( \frac{1}{\sqrt{n}}\right) ,
\end{align*}
and 
\begin{align*}
\frac{1}{n} {\bf A}^\top {\bf b}- \frac{1}{n} {\bf B}^\top {\bf b} &= \frac{1}{n} {\bf F}^\top {\bf b} \\
&=\sigma^2 \frac{1}{n}  {\bf G}^\top {\bf 1}_{2n \times 1} +O_p\left( \frac{1}{\sqrt{n}}\right).
\end{align*}
Therefore, 
\begin{align*}
\hat {\bm \theta}_n^{\rm BE} &=\left( {\bf A}^\top {\bf A}-\hat \sigma^2 {\bf G}^\top {\bf G}\right) ^{-1} \left( {\bf A}^\top {\bf b}-\hat \sigma^2  {\bf G}^\top {\bf 1}_{2n \times 1} \right) \\
& = \left( \frac{1}{n} {\bf B}^\top {\bf B}+O_p\left(\frac{1}{\sqrt{n}} \right)  \right) ^{-1} \left( \frac{1}{n} {\bf B}^\top {\bf b}+O_p\left(\frac{1}{\sqrt{n}} \right) \right) \\
& = \hat {\bm \theta}_n^{\rm UB} +O_p\left(\frac{1}{\sqrt{n}} \right). 
\end{align*}
Since $\hat {\bm \theta}_n^{\rm UB}$ is $\sqrt{n}$-consistent, so is $\hat {\bm \theta}_n^{\rm BE}$.

%
%
%
%
%

\ifCLASSOPTIONcaptionsoff
  \newpage
\fi

\small
\bibliographystyle{IEEEtran}
\bibliography{sj_reference}
\end{document}